\documentclass[11pt]{article}

\usepackage{amssymb}
\usepackage{amsmath,amsfonts,amsthm}
\usepackage{latexsym}
\usepackage{mathrsfs}
\usepackage{color}
\usepackage{algorithm}
\usepackage{algorithmic}
\usepackage{subcaption}
\usepackage{dsfont}
\usepackage[overload]{empheq}

\usepackage{hyperref}       
\usepackage{cleveref} 
\usepackage{url}            
\usepackage{booktabs}       
\usepackage[a4paper]{geometry}
\usepackage{bm}
\usepackage{graphicx}
\usepackage[numbers]{natbib}
\usepackage{makecell}

\newcommand{\beqa}{\begin{eqnarray}}
\newcommand{\eeqa}{\end{eqnarray}}
\newcommand{\beqas}{\begin{eqnarray*}}
\newcommand{\eeqas}{\end{eqnarray*}}
\newcommand{\ba}{\begin{array}}
\newcommand{\ea}{\end{array}}
\newcommand{\bi}{\begin{itemize}}
\newcommand{\ei}{\end{itemize}}

\newcommand{\RN}[1]{%
  \textup{\uppercase\expandafter{\romannumeral#1}}%
}

\newtheorem{lemma}{Lemma}
\newtheorem{theorem}{Theorem}

\newtheorem{definition}{Definition}
\newtheorem{proposition}{Proposition}

\newtheorem{assumption}{Assumption}

\newtheorem{example}{Example}

\newcounter{spb}
\newcommand{\email}[1]{\protect\href{mailto:#1}{#1}}

\def\b0{\bm{0}}
\def\b1{\bm{1}}
\def\ba{\bm{a}}

\def\bu{\bm{u}}

\def\bA{\bm{A}}

\def\bD{\bm{D}}
\def\bE{\bm{E}}
\def\tE{\tilde{\bm{E}}}

\def\bH{\bm{H}}

\def\bI{\bm{I}}

\def\bM{\bm{M}}

\def\bP{\bm{P}}

\def\bQ{\bm{Q}}
\def\bR{\bm{R}}

\def\bU{\bm{U}}
\def\tU{\tilde{\bm{U}}}
\def\bV{\bm{V}}

\def\bDelta{\bm{\Delta}}
\def\tDelta{\tilde{\bm{\Delta}}}

\def\mI{\mathcal{I}}

\def\P{\mathbb{P}}
\def\R{\mathbb{R}}

\def\oU{\overline{\bm U}}

\begin{document}
\title{Symmetric Matrix Completion with ReLU Sampling\thanks{This work has been accepted for publication in the Proceedings of the 41st International Conference on Machine Learning (ICML 2024). The first and second authors contributed equally to this work. Correspondence to: Laura Balzano (\email{girasole@umich.edu}).}}
\author{
Huikang Liu\thanks{Antai College of Economics and Management, Shanghai Jiao Tong University, Shanghai}
\and
Peng Wang\thanks{Department of Electrical Engineering and Computer Science, University of Michigan, Ann Arbor}
\and  
Longxiu Huang\thanks{Department of Computational Mathematics, Science and Engineering, Michigan State University, Lansing}
\and
Qing Qu\thanks{Department of Electrical Engineering and Computer Science, University of Michigan, Ann Arbor}
\and
Laura Balzano\thanks{Department of Electrical Engineering and Computer Science, University of Michigan, Ann Arbor}
}
\maketitle

\begin{abstract}

We study the problem of symmetric positive semi-definite low-rank matrix completion (MC) with deterministic entry-dependent sampling. In particular, we consider rectified linear unit (ReLU) sampling, where only positive entries are observed, as well as a generalization to threshold-based sampling. We first empirically demonstrate that the landscape of this MC problem is not globally benign: Gradient descent (GD) with random initialization will generally converge to stationary points that are not globally optimal. Nevertheless, we prove that when the matrix factor with a small rank satisfies mild assumptions, the nonconvex objective function is geodesically strongly convex on the quotient manifold in a neighborhood of a planted low-rank matrix. Moreover, we show that our assumptions are satisfied by a matrix factor with i.i.d. Gaussian entries. Finally, we develop a tailor-designed initialization for GD to solve our studied formulation, which empirically always achieves convergence to the global minima.  We also conduct extensive experiments and compare MC methods, investigating convergence and completion performance with respect to initialization, noise level, dimension, and rank.  

\end{abstract} 

\section{Introduction}\label{sec:intro}

Low-rank matrix completion (MC) refers to the task of filling in the missing entries of a partially observed low-rank matrix. It has found applications in diverse fields, such as recommendation systems \cite{koren2009matrix}, sequential bioinformatics \cite{zheng2013collaborative}, and computer vision \cite{ji2010robust}, to name a few. 
In particular, symmetric positive semi-definite (PSD) low-rank MC has applications in covariance matrix completion \cite{candes2010matrix,hosseini2017array}, Hankel matrix completion \cite{chen2013spectral,usevich2016hankel,cai2023structured}, and Euclidean distance matrix completion \cite{laurent2001polynomial, al2005approximate,dokmanic2015euclidean}. An extensive body of literature investigates the statistical and optimization properties of the low-rank MC problem using different approaches, such as nuclear-norm minimization \cite{recht2011simpler,candes2010matrix,candes2012exact}, spectral method \cite{keshavan2010matrix,chatterjee2015matrix}, and matrix factorization \cite{sun2016guaranteed}, among others. 
To analyze the MC problem, most of these approaches rely on the following assumptions: (1) the underlying matrix is low-rank and incoherent, and (2) the entries are observed \emph{independent of the matrix entries} according to a probabilistic mechanism, e.g., uniformly at random. Based on these assumptions, it is possible to prove how many observed entries are required so that the missing entries can be exactly or approximately completed. 

The majority of existing work assumes that the entries are observed uniformly at random, independent of the underlying matrix values. However, this assumption is strict and often violated in practice. In real-world applications where data are collected from measurements, such as distance matrices, missing entries tend to be those that are harder to collect. 
When data are collected from
participants, such as online shopping or surveys, missing answers are typically highly correlated with the question and the value of the true answer. For example, surveys about sensitive topics will have missing entries on any culturally or morally problematic answers. Moreover,
online ratings tend to be skewed to the high end; for example, most people do not read a book or watch a movie that they might dislike. 
Even in sensor systems, missing data are not likely to be independent and completely random. Sensors may saturate at a certain value or break down based on environmental conditions that also affect the sensor value. In all these applications, the probability of missing entries in a matrix is dependent on the underlying values, sometimes deterministically.  

Despite being highly relevant for applications, the problem of recovering missing entries when the sampling mechanism depends on the entries remains a challenging and relatively under-explored area of research. 
Existing papers along these lines give impractical results, such as those where the recovery metric provides no guarantees for recovering entries that are never observed \cite{foygel2011learning,foucart2020weighted}, or high-dimensional consistency results that do not admit finite sample guarantees or make strong assumptions on the sampling probability functions, for example that they are Lipschitz in the matrix entry value and nonzero everywhere \cite{bhattacharya2022matrix}. 
Research works on MC with deterministic sampling focused on understanding fundamental properties that the sampling patterns must exhibit \cite{lee2023matrix}. 
However, none of these works provide clear and practical guidelines for what MC problems can be solved in their given settings and for what algorithms will successfully complete them. 
 
\subsection{Our Contributions}

Our work aims to advance the understanding of MC with deterministic sampling by focusing on symmetric PSD low-rank MC with ReLU sampling, where only the positive entries of a matrix are observed. In this setting, under relatively mild assumptions, we first show that the globally optimal solutions of the low-rank MC problem with a known rank can exactly (resp. approximately) complete the missing entries in the noiseless (resp. noisy) setting. 
Moreover, we prove that the objective function is geodesically strongly convex on the quotient manifold around the {\em planted} low-rank matrix, i.e., the underlying true low-rank matrix. Therefore, with an initial point close enough to the planted low-rank matrix, GD will converge to this desired matrix. This motivates us to tailor an initialization method for GD. 
We empirically demonstrate that GD with our initialization always converges to the planted low-rank matrix. It is worth mentioning that this objective function landscape result is the first of its kind within the literature on dependent or deterministic sampling for MC.

While ReLU sampling is a specific setting, it has the potential to be generalized to a broad class of common missing data problems, e.g., where only a range of values is observed.  In this direction, we have also provided more general assumptions for a broader class of sampling functions. For example, a threshold sample where we observe all entries above a positive threshold will also lead to a strongly convex objective around the planted low-rank matrix. This more general setting requires stronger assumptions, but it gives us more general results that also apply to the noisy setting.
We note that this thresholding sampling setting generally means that a constant fraction of entries are observed. This is an interesting sampling regime when observations are entry-dependent, and a useful one for many sensor data applications where it may be feasible to collect a moderate number of matrix entries.

Even in the best case of our theorems, empirical results (shown in \Cref{sec:expe}) drastically outperform the theory, leaving an exciting open question as to whether better assumptions and proof techniques can deliver theory that matches empirical observation. 
 
\subsection{Literature Review}

%

We review three categories of work that are closed related to our work: deterministic sampling, probabilistic sampling that depends on entry values, and general matrix completion and factorization literature.

\vspace{-0.1in}
\paragraph{Deterministic sampling.} Within deterministic sampling, there is work that must make incoherence or genericity assumptions on the underlying matrix, and work that removes almost all assumptions on the matrix but either needs to leverage other properties (i.e. PSD) or ignore problematic parts of the matrix. We will start with work that does not make assumptions on the underlying matrix.

Possibly most similar to our work is \cite{bishop2014deterministic}, 
where the authors consider deterministic sampling of PSD matrices. They require sampling of principal submatrices, which is very similar to the assumptions we have made, since under ReLU sampling, we observe the diagonal, and under the partition/permutation of \eqref{eq:Upartition}, we also observe the diagonal blocks. 
To the best of our knowledge, neither of the assumptions in our paper or in \cite{bishop2014deterministic} implies the other. More specifically, \citet{bishop2014deterministic}  consider deterministic sampling and assume that a collection of subsets of $\Omega$ admit an ordering such that any two adjacent parts have enough overlap. 
However, in the ReLU sampling setting, this order is hard to construct, and it is not clear whether there exists such an ordered collection of subsets. Moreover, both the analysis and the algorithm proposed in \cite{bishop2014deterministic} heavily rely on these ordered subsets, while we only use our partition for analysis.
They design an algorithm for completion based on their theoretical guarantees, whereas our work simply uses the well-known gradient descent algorithm for completion.
Additionally, to the best of our knowledge, it is not possible to extend their approach to non-PSD matrices.
Our approach also relies heavily on the synergy of PSD matrices and ReLU sampling, but we believe that the generalization we provide in Assumption \ref{ass:main}, and further generalizations thereof, can break this dependency. 

\citet{mazumdar2018representation} studies the ReLU recovery problem in the context of neural network parameters. 
They provide theoretical guarantees for estimating a low-rank matrix observed after adding a bias vector and passing through a ReLU. They formulate a likelihood based on the distribution of the bias vector and show that the maximizer of the likelihood is close to the planted low-rank matrix.
However, it is not clear how to solve their proposed likelihood. 
Their work was inspired in part by the related work in \cite{soltanolkotabi2017learning}, which uses gradient descent to estimate vectors that are observed through a known linear operator and then passed through a ReLU, and provides high-probability finite sample guarantees on the accuracy of the estimates when the known linear measurement operator is random. The matrix completion problem in \cite{ganti2015matrix} seeks to estimate a partially observed matrix, which is a low-rank matrix observed through an unknown entry-wise monotonic and Lipschitz nonlinearity (like ReLU). They estimate both the nonlinearity and the low-rank matrix in an alternating fashion. Their guarantees are for estimating the matrix entries after the nonlinearity, as opposed to the latent low-rank matrix entries directly. 

\citet{foucart2020weighted} studies the non-symmetric matrix completion problem when the observation pattern is deterministic.  It introduces a methodology for deriving a weighting matrix tailored to the specific sampling pattern presented.  They then provide an efficient initialization scheme by making good use of the weighting matrix for the matrix completion problem. Their recovery guarantees are with respect to this weighting matrix, so if parts of the true matrix are not recoverable, the corresponding part of the weight matrix will be zero. This allows them to avoid assumptions on the underlying matrix (other than standard assumptions on the maximum value of the matrix). Other works in this direction include \cite{chatterjee2020deterministic}, which proves necessary asymptotic properties of the sampling patterns for a notion of ``stable recovery'' when no assumptions are made on the underlying matrix. The requirements are strong and general, and so the results are somewhat limited, in particular, their results imply that no sparse deterministic pattern can guarantee stable recovery.

For works that make assumptions on the underlying matrix, 
\citet{pimentel2016characterization} assumes the matrix columns (or rows) are in general position and gives nearly necessary assumptions on the deterministic pattern. \citet{kiraly2015algebraic} also assumes genericity and makes stronger assumptions. 
\citet{liu2017new, liu2019matrix} develops a different assumption related to the spark of the sensing matrix or observation pattern and shows that the planted matrix uniquely fits the entries. 
\citet{shapiro2018matrix} develops conditions for a low-rank solution to be locally unique, i.e. the unique solution in a neighborhood around that point, and makes several interesting connections to the above literature. Finally, \citet{singer2010uniqueness} makes a connection to rigidity theory and provides a method to determine whether a unique completion of a given partially observed matrix and a given rank is possible.


\vspace{-0.1in}
\paragraph{Probabilistic entry-dependent sampling.} 
There is a line of work in matrix completion that considers arbitrary sampling distributions \cite{foygel2011learning, shamir2014matrix} and seeks to bound the expected loss, where the expectation is taken with respect to the sampling distribution.
Therefore, distinct from our work, if an entry is observed with probability zero, their results provide no guarantees for recovery. In classical statistics literature,  one models the data as a random variable and the missingness mechanism as random, represented by a binary random variable, and potentially dependent on the data. In this context a commonly used taxonomy of missing data mechanisms follows \cite{little2019statistical}, which defines MCAR, MAR, and MNAR: (i) MCAR, missing completely at random, where the missing data variable is independent of the data, e.g. entries missing uniformly at random in classical matrix completion literature fits into this category; (ii) MAR, missing at random where the missing data variable is independent of the unobserved data when conditioned on that which is observed (data and/or covariates); and (iii) MNAR, missing not at random where the missing data variable depends on that which is unobserved, even conditioned on that which is observed.

Using this probabilistic setup and terminology, \citet{hernandez2014probabilistic} proposes a Bayesian approach to jointly infer a complete data factorization model and a missing data mechanism, also modeled with matrix factorization. They also provide approximate inference methods for the resulting intractable Bayesian inference problem.
\citet{sportisse2020estimation} considers the identifiability of PPCA with MNAR data and provides 
an estimation algorithm for the principal components in this setting. \citet{sportisse2020imputation} proposes interesting EM-style methods to estimate the joint distribution of the missing data mechanism and the data. \citet{ma2021identifiable} identifies novel sufficient conditions such that the ground truth parameters of a parametric data model are identifiable, and maximum likelihood identifies them uniquely, in the MNAR setting. 

Other papers focused on the MNAR setting include \cite{yang2021tenips,jin2022matrix, agarwal2023causal}. \citet{sengupta2023sparse} compares matrix factorization to multiple imputation, the most popular framework in biostatistics and social sciences for imputing missing entries in the MAR setting. An empirical work that includes ReLU sampling with other MNAR sampling schemes is found in \cite{naik2022truncated}. The authors conclude that convex methods generally do not work as well as nonconvex in the dependent-sampling setting. Closely related to our work is also \cite{bhattacharya2022matrix}, which assumes that each entry of a low-rank $\bM^*$ is observed with some probability $f(\bM^*)$, where $f$ is applied entry-wise and is essentially Lipschitz and non-zero. 
The authors provide modified singular value thresholding and nuclear norm estimators for this setting and prove consistency in the high-dimensional regime, where the dimensions of $\bM^*$ grow but the rank remains uniformly bounded. A similar estimator is proposed in \cite{ma2019missing}, and under a low-rank assumption on the matrix of probabilities of observation, they prove error bounds for estimating the matrix of observation probabilities as well as the underlying matrix.

\vspace{-0.1in}
\paragraph{More general factorization and completion literature.}
The work in \cite{saul2022nonlinear} is highly related to our work, though it is not a matrix completion problem per se. 
This paper seeks a non-negative low-rank matrix $\mathbf{X}$ such that very sparse observation $\bM = f(\mathbf{X})$ and $f$ is a given nonlinearity, such as ReLU. They provide an EM algorithm for learning $\mathbf{X}$ but no theoretical guarantees for when this low-rank approximation exists. 
The work in \cite{seraghiti2023accelerated} studies the same problem and provides an alternating block coordinate descent approach. 

Finally, the work in \cite{ongie2020tensor} studies matrix completion with uniformly missing entries, but the underlying matrix columns are points on a low-dimensional nonlinear variety. The algorithm for completion lifts the partially observed data using a polynomial lifting, which creates a non-uniform sampling structure in the lifted space, where an entire deterministic set of entries in the lifted space must be missing if one entry is missing in the original space. The authors show that it is still possible to perform completion with this structured missing pattern under certain genericity conditions.

\vspace{-0.1in}
\paragraph{Notation.} We write matrices in bold capital letters like $\bm A$, vectors in bold lower-case letters like $\bm{a}$, and scalars in plain letters. Given a matrix $\bm A \in \R^{m\times n}$, we use $\sigma_{\max}(\bm A)$ or $\|\bm A\|$ to denote its largest singular value (i.e., spectral norm), $\sigma_{\min}(\bm A)$ its smallest non-zero singular value, $\|\bm A\|_F$ its Frobenius norm, and $a_{ij}$ its $(i,j)$-th element. Given a vector $\bm a \in \R^n$, we denote its Euclidean norm by $\|\bm a\|$ and the $i$-th entry by $a_i$. Given a positive integer $n$, we denote by $[n]$ the set $\{1,\ldots,n\}$. Let $\mathcal{O}^{m\times n} = \left\{\bm Q \in \R^{m\times n}: \bm Q^T\bm Q = \bm I_n \right\}$ denote the set of all $m\times n$ orthonormal matrices. In particular, let $\mathcal{O}^m = \{\bQ \in \mathbb{R}^{m \times m}: \bQ^T\bQ = \bI_m\}$ denote the set of all $m\times m$ orthogonal matrices. Given an integer n, we define $[n] := \{1,\dots,n\}$. 

\section{Problem Formulation}\label{sec:setup}


\paragraph{MC with ReLU sampling.} In this work, we consider a noisy symmetric PSD MC completion problem. Specifically, let $\bm M^\star := \bm{U}^\star \bm{U}^{\star T} \in \R^{n\times n}$ be a symmetric PSD matrix, where $\bm{U}^\star \in \R^{n \times r}$. In addition, let $\bm M \in \R^{n \times n}$  be a noisy version of $\bm M^\star$ generated by
\begin{align}\label{model:UV}
\bm{M} := \bm M^\star + \bm \Delta,
\end{align}
where $\bm \Delta \in \R^{n\times n}$ is a noise matrix. 
It is worth noting that $\bm M$ is non-symmetric when the noise matrix $\bm \Delta$ is not symmetric. 

There are many applications
where it is common to observe only a partial set of the entries of $\bm M$ \cite{nguyen2019low}. This could be due to data collection, experimental constraints, or inherent missing information \cite{hu2008collaborative}. In this work, we consider a setting where the missingness pattern of the matrix is dependent on the underlying values in the matrix and is deterministic given the matrix entries. Specifically, we suppose that only the non-negative entries in $\bm M$ can be observed, i.e., the observed set 
\begin{align}\label{eq:Ob}
\Omega = \left\{(i,j) \in [n] \times [n]: m_{ij} \ge 0 \right\}. 
\end{align}
Notably, this sampling regime is commonly referred to as ReLU sampling in the literature (see, e.g., \citet{naik2022truncated,mazumdar2018representation}), as it utilizes the function $f(x) = \max\{0,x\}$, known as the rectified linear unit (ReLU) function in deep learning\footnote{See \cite{nair2010rectified} for early use of the phrase ``rectified linear unit,'' but it had been in use for neural nets long before, referred to by the name ``linear threshold unit'' \cite{wersing2001dynamical} or ``positive part'' \cite{jarrett2009best}, among others.}. Then, our goal is to complete the missing entries of $\bm M^\star$ from the observed entries in $\bm M_{\Omega}$.

\begin{figure}[t]
\begin{center}
	\begin{minipage}[b]{0.45\linewidth}
		\centering
		\centerline{\includegraphics[width=1\linewidth]{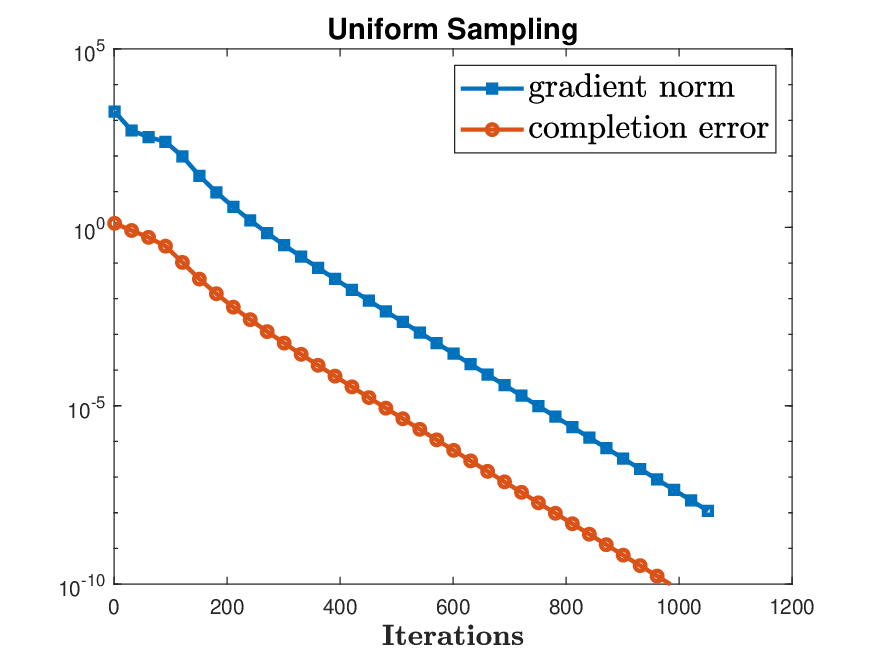}}
	\end{minipage}\vspace{-0.1in}
	\begin{minipage}[b]{0.45\linewidth}
		\centering
		\centerline{\includegraphics[width=1\linewidth]{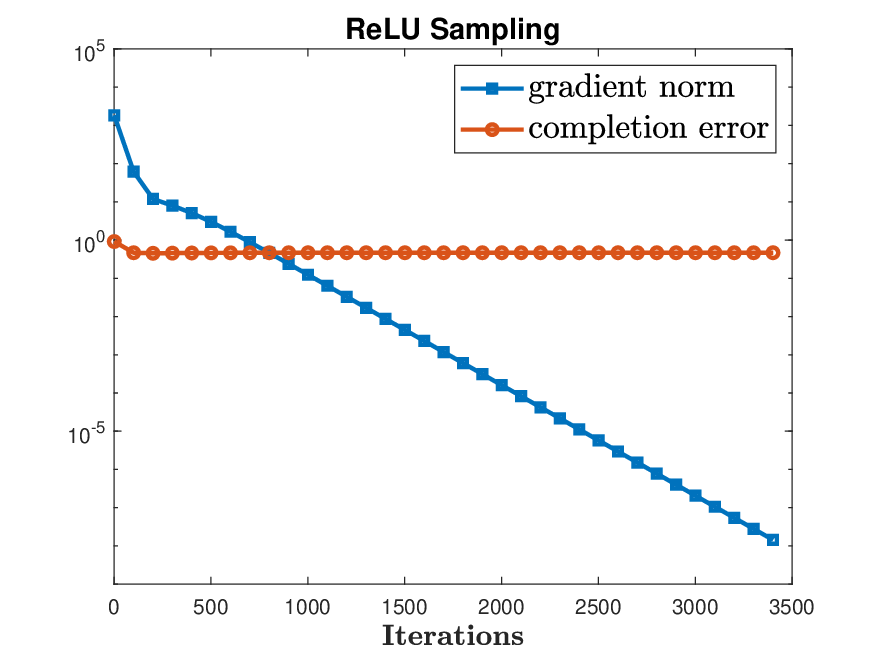}}
	\end{minipage}\vspace{-0.1in}
\end{center}
	\vspace{-0.1in}
	\caption{{\bf Recovery and convergence performance of GD for solving the MC problem with the uniform $(p=0.2)$ and ReLU sampling in the noiseless case.} We apply GD with Gaussian random initialization for solving Problem \eqref{eq:MC} with the uniform and ReLU sampling, respectively. Then, we plot the gradient norm (i.e., $ \|\nabla F(\bm U^{(t)}) \|_F$) and completion error (i.e., $\| \bm U^{(t)}\bm U^{(t)^T} - \bm M \|_F/\|\bm M\|_F$) against number of iterations.} \vspace{-0.1in}
   \label{fig:crit}
\end{figure}  

Before proceeding, we make some remarks on this MC problem. First, existing works \cite{naik2022truncated, saul2022nonlinear, seraghiti2023accelerated} have empirically studied the low-rank MC problem with ReLU sampling and focused on proposing efficient algorithms to address this problem. In particular, \citet{mazumdar2018representation} has theoretically studied the recovery performance of the ReLU-based representation learning problem under a probabilistic model. This differs from our work, which studies the global optimality and optimization landscape of the MC problem. Second, this sampling is merely a specific instance within a broader set of general deterministic sampling schemes.
We prove results on one such generalization, where we observe values above or below a threshold.
We believe our work will be a springboard for the study of more general practical entry-dependent sampling schemes. 

\vspace{-0.1in}
\paragraph{Optimization formulation.} Leveraging the low-rank structure in \Cref{model:UV}, we consider the following non-convex MC problem to complete the missing entries of $\bm M^*$:
\begin{align}\label{eq:MC}
    \min_{\bm{U} \in \R^{n\times r}} F(\bm U) := \frac{1}{4} \left\|\left(\bm U\bm U^T - \bm M \right)_{\Omega} \right\|_F^2.
\end{align}
Note that the rank of the optimization variable $\bm U$ is exactly the rank of the planted low-rank matrix $\bm M^*$. In the noiseless setting with uniform sampling and incoherence assumptions, \citet{ge2016matrix} have shown that Problem \eqref{eq:MC} has a benign global optimization landscape in the sense that it has no spurious local minima---all local minima must also be global minima. 
This, together with the result in \cite{lee2016gradient}, implies that gradient descent (GD) with random initialization with high probability converges to globally optimal solutions that achieve exact completion; see \Cref{fig:crit}(a). One may then conjecture that Problem \eqref{eq:MC} in the ReLU sampling setting also has such a benign optimization landscape. However, this conjecture is 
refuted by empirical evidence in \Cref{fig:crit}(b). This result is typical in the ReLU sampling setting, illustrating that GD with random initialization may converge to a spurious critical point that is not globally optimal. 
The next question is whether it is even possible to recover the missing entries by applying GD for this setting. In this work, we answer this question in the affirmative.  

\section{Main Results}\label{sec:main}

In this section, we first present our theoretical result on symmetric PSD low-rank MC with ReLU sampling in the noiseless case under mild assumptions in \Cref{subsec:noiseless}. 
In \Cref{subsec:noisy}, we extend these results to the noisy case under more general sampling assumptions. In \Cref{subsec:Gaus}, we show that all the introduced assumptions hold with high probability when the entries of $\bm U^\star$ are i.i.d. sampled from the standard Gaussian distribution.   

\subsection{Noiseless Case and ReLU Sampling}\label{subsec:noiseless}

In this subsection, we consider the noiseless case, i.e., $\bDelta = \bm 0$, and ReLU sampling in \eqref{eq:Ob}. We start by introducing some assumptions on the underlying matrix $\bm U^\star$ and the sampling set $\Omega$. Noting that the $r$-dimensional space has $2^r$ orthants, we denote by $\mathcal{C}_i$ the $i$-th orthant of the $r$-dimensional space for each $i \in [2^r]$. For example, there are $4$ orthants (i.e., quadrants) in the $2$-dimensional space and $8$ orthants (i.e., octants) in the $3$-dimensional space. Here, $\mathcal{C}_i$ for each $i \in [2^r]$ is ordered such that the signs of components of a vector belonging to orthant $\mathcal{C}_i$ differ from those in orthant $\mathcal{C}_{i+1}$ in only one component. For ease of exposition, without loss of generality, we assume the rows of $\bm U^\star \in \R^{n\times r}$ are partitioned into the following blocks
\begin{align} \label{eq:Upartition}
    \bm U^\star = \begin{bmatrix}
        \bm U_1^{\star T} & \bm U_2^{\star T} & \dots & \bm U_{2^r}^{\star T}
    \end{bmatrix}^T \in \R^{n\times r},
\end{align}
where each row of $\bm U_i^\star \in \R^{n_i\times r}$ belongs to the $i$-th orthant $\mathcal{C}_i$ for each $i \in [2^r]$ and $\sum_{i=1}^{2^r} n_i = n$. For example,  when $r=2$, the rows of different blocks of $\bm U^*_i$ takes the signs as shown in \Cref{fig:sign}. Based on the above setup, we make the following assumption.  

\begin{assumption}\label{ass:uistar}
     For each $i \in [2^r]$, ${\rm rank}(\bU_i^\star) = r$.
\end{assumption}

We remark that if $r$ is much smaller than $\log n$ and the entries of $\bU^\star$ are i.i.d. sampled from a distribution symmetric about zero, such as the standard normal distribution, 
then the generated submatrices are full rank with high probability.  

\begin{figure}[t]
\begin{center}
\centerline{\includegraphics[width=0.6\columnwidth]{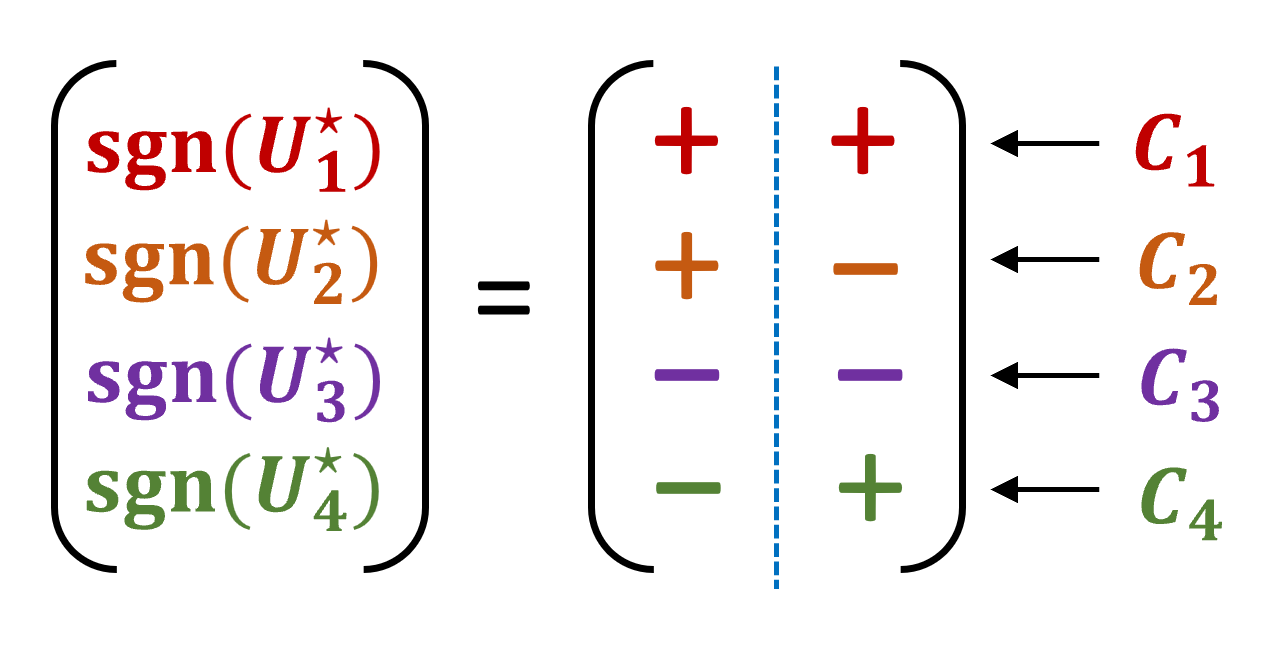}}\vspace{-0.2in}
\caption{{\bf An illustrative figure on the partition of rows of $\bm U^\star \in \R^{n\times 2}$.} We rearrange the rows of $\bm U^\star$ and partition them into 4 blocks, each belonging to different orthants.}\label{fig:sign}
\end{center}
\vskip -0.3in
\end{figure} 
 
Next, we discuss the second assumption that will lead to unique completion. This assumption is illustrated in \Cref{fig:samping}. For any pair of indices $(i, j)\in [2^r] \times [2^r]$, we denote by $\bM^{(i,j)} \in \mathbb{R}^{n_i \times n_j}$ the $(i,j)$-th block of $\bM$, i.e., $\bM^{(i,j)} = \bU^{\star}_{i} \bU^{\star T}_{j}$. Moreover, we denote the set of observations in the $(i,j)$ block by $\Omega_{i,j} = \{ (k, l) \in [n_i] \times [n_j]: m^{(i,j)}_{kl} \geq 0 \}$ 
where $m^{(i,j)}_{kl}$ is the $(k,l)$-th entry of $\bM^{(i,j)}$. Let a column vector $\bu^{\star}_{i, k} \in \R^{r}$ denote the $k$-th row of $\bU^\star_{i}$ for each $k \in [n_i]$. Then, we have $m^{(i,j)}_{kl} = \bu^{\star T}_{i, k}\bu_{j,l}^{\star}$. It is obvious that $\Omega_{i,i} = [n_i] \times [n_i]$ under the ReLU sampling since $\bu^{\star T}_{i,k} \bu^{\star}_{i,l} \geq 0$ always holds for all $k,l \in [n_i]$ due to the fact that the rows of $\bU^\star_i$ have the same sign. 
Besides, since $\bM^{(i+1, i)} = \bU^{\star}_{i+1}\bU_i^{\star T}$ and the signs of components in each row of $\bU^\star_{i+1}$ and those of $\bU^\star_{i}$ only differ in one component, one can image that there are enough observations in $\Omega_{i+1, i}$ with ReLU sampling. More precisely, we can formalize the above observation as follows: 

\begin{assumption}\label{ass:omegai}
    For any $i \in [2^r - 1]$, we have $|\Omega_{i+1,i}| \geq r^2$ and the matrix space spanned by $\{ \bu^{\star}_{i+1, k} \bu^{\star T}_{i, l} :(k, l) \in \Omega_{i+1,i} \}$ is the whole space $\mathbb{R}^{r \times r}$. 
\end{assumption}

Assumptions \ref{ass:uistar} and \ref{ass:omegai} guarantee that there are enough positive (observed) entries for the uniqueness of the completion. We note that these assumptions can be checked from the observed matrix by finding a permutation of rows and columns such that the diagonal blocks are fully observed and of rank $r$ and the off-diagonal blocks have enough entries.  In Section \ref{subsec:Gaus}, we will show that, when $\bU^\star$ is i.i.d. Gaussian random matrix and $r \leq \frac12 \log n$, Assumptions \ref{ass:uistar} and \ref{ass:omegai} hold with high probability.

\begin{figure}[t]
\centerline{\includegraphics[width=0.6\columnwidth]{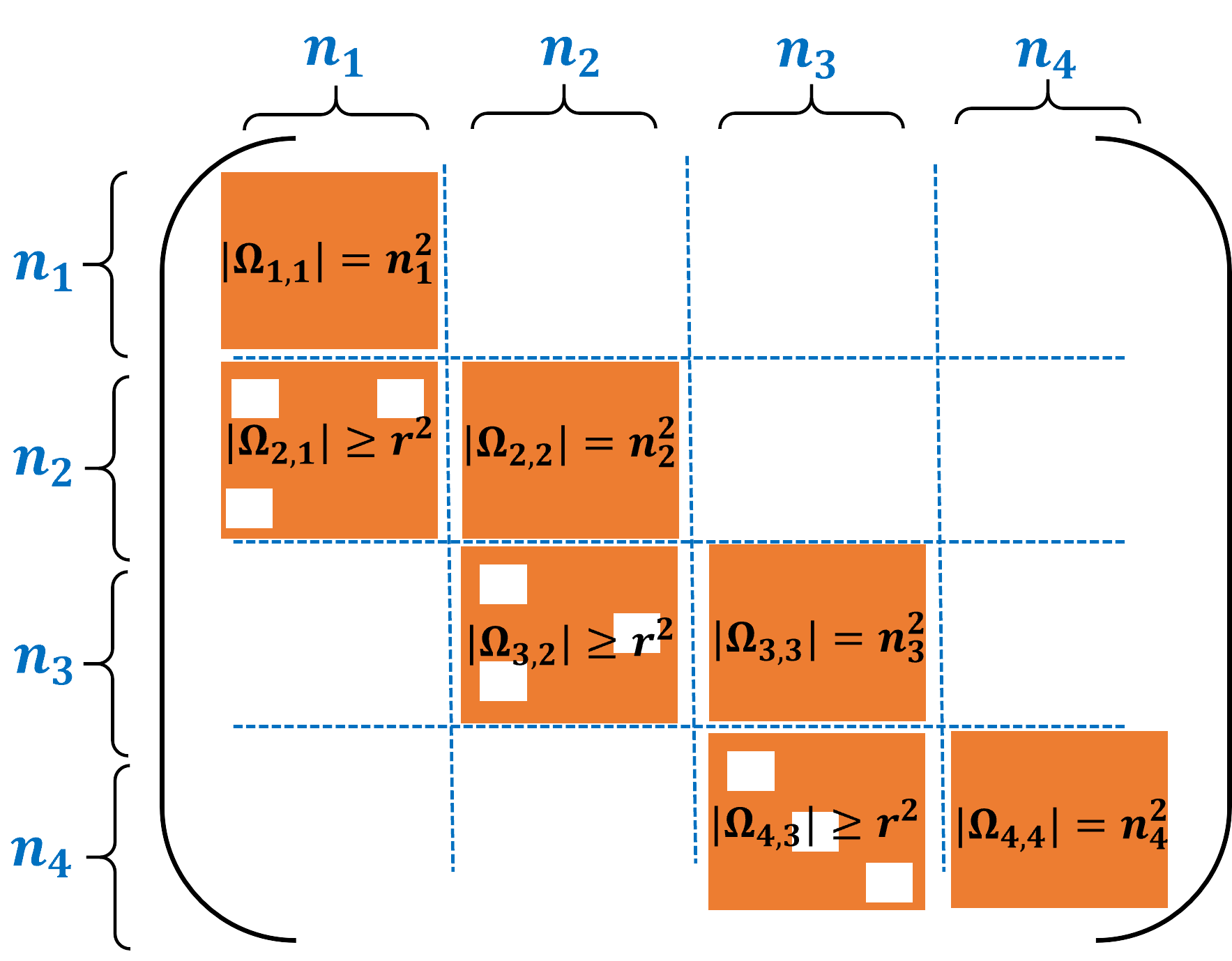}}\vspace{-0.1in}
\caption{{\bf An illustrative figure on \Cref{ass:omegai} when $r = 2$.} Orange pixels denote observed entries, while white pixels denote missing entries.}\label{fig:samping}
\end{figure}

\vspace{-0.1in}
\paragraph{Characterization of global optimality.} Based on the two assumptions, we are ready to characterize the global optimality of the MC problem \eqref{eq:MC} under ReLU sampling. 

\begin{theorem}\label{thm:noiseless:1}
Suppose that $\bm \Delta = \bm 0$ in \eqref{model:UV}, the observed set $\Omega$ is defined in \eqref{eq:Ob}, and Assumptions \ref{ass:uistar} and \ref{ass:omegai} hold. Then, $\bm U \in \R^{d\times r}$ is a global optimal solution to Problem \eqref{eq:MC} if and only if it satisfies $\bm U\bm U^{T} = \bm M^\star$. 
\end{theorem}
We defer the proof of this theorem to \Cref{pf:thm 1}. This theorem demonstrates that even under the ReLU sampling regime, the global optimal solutions of Problem \eqref{eq:MC} still recover the underlying matrix $\bm M^\star$ exactly, just as in the uniform sampling regime. 
\vspace{-0.1in}
\paragraph{Uniqueness of global solutions on the manifold.} Remark that the objective function $F(\cdot)$ of Probelm \eqref{eq:MC} is invariant under any orthogonal matrices in $\mathbb{R}^{r \times r}$, i.e., $F(\bU \bQ) = F(\bU)$ for all $\bQ \in \mathcal{O}^r$. To characterize this invariance to orthogonal transformations, let $\sim$ denote an equivalent relation on $\mathbb{R}^{n\times r}$ with the equivalence class 
\begin{align*}
    [\bU] := \left\{\bV \in \mathbb{R}^{n\times r} : \bV \sim \bU \text{ iff }\ \exists\bQ \in \mathcal{O}^r, \bV = \bU \bQ\right\}.
\end{align*}
Then, we define the Riemannian quotient manifold $\mathcal{M}$ as
\begin{align}\label{eq:mani}
    \mathcal{M} \ :=\ \mathbb{R}^{n\times r} \setminus \sim  \ = \ \{[\bU]: \bU \in \mathbb{R}^{n\times r}\}.
\end{align}
If we consider Problem \eqref{eq:MC} on the manifold $\mathcal{M}$, \Cref{thm:noiseless:1} implies that Problem \eqref{eq:MC} has a unique global optimal solution $[\bU^\star]$. Next, we will show that the objective function $F(\cdot)$ exhibits geodesic strong convexity on the quotient manifold $\mathcal{M}$ near the global optimum $[\bU^\star]$ .  
 
\vspace{-0.1in}
\paragraph{Main technical ingredient.}  The proof of \Cref{thm:noiseless:1} relies on the following technical lemma, which could be independent interest. This lemma indicates that the distance between two low-rank matrices on the quotient manifold $\cal M$ is bounded by their subspace distance. Its proof is also deferred to \Cref{sec:app:techlem}. 

\begin{lemma}\label{ineq:tech-lemma} 
   For arbitrary $\bU, \bV \in \R^{n\times r}$ with $r \le n$, suppose that $\text{rank}(\bU) = r$. Then, there exists an orthogonal matrix $\bQ \in \mathcal{O}^{r}$ such that 
   \begin{align*}
        \|\bU - \bV\bQ\|_F \leq \frac{\sigma_{\min}(\bm U) + \|\bV\|}{\sigma_{\min}^2(\bm U)}  \|\bU\bU^T - \bV\bV^T\|_F. 
   \end{align*}
\end{lemma} 
\vspace{-0.1in} 

\vspace{-0.1in} 
\paragraph{Preliminary setup for manifold optimization.}  As discussed above, we study Problem \eqref{eq:MC} over the quotient manifold $\cal M$ defined in \eqref{eq:mani}. Toward this end, we introduce some concepts of manifold optimization, such as tangent space, Riemannian gradient, and Riemannian Hessian. Since the formal definition of the {\em tangent space} to a quotient manifold is abstract, we describe it informally as follows. We define the {\em vertical space} at $\bU \in \mathbb{R}^{n\times r}$ denoted by $\mathcal{S}_{\bU}$ (see \citet[Definition 9.23, Example 9.25]{boumal2023introduction}) as
\begin{align}\label{defi:verticalspace}
    \mathcal{S}_{\bU} := \left\{ \bU\bR \in \R^{n\times r}: \bR + \bR^T = \bm 0,  \bR \in \mathbb{R}^{r\times r}\right\}.
\end{align}
As shown in \eqref{eq:form-sperp}, the orthogonal complement of $\mathcal{S}_{\bU}$, which is called the horizontal space and denoted by $\mathcal{S}^\perp_{\bU}$, is  
\begin{align}\label{defi:horizontalspace}
    \mathcal{S}^\perp_{\bU} := \left\{\bD \in \mathbb{R}^{n \times r} : \bU^{T} \bD = \bD^T \bU \right\}.
\end{align}
According to \citet[Definition 9.24]{boumal2023introduction}), for any $\bU \in \mathbb{R}^{n\times r}$, there exists a bijective mapping $\mathrm{lift}_{\bU}$ between any tangent vector $\bm \xi \in \mathrm{T}_{[\bU]} \mathcal{M}$ and any matrix $\bm D \in \mathcal{S}^\perp_{\bU}$, i.e., 
\[
\mathrm{lift}_{\bU}: \mathrm{T}_{[\bU]} \mathcal{M} \mapsto \mathcal{S}^\perp_{\bU},  \, \mathrm{lift}_{\bU}(\bm \xi) = \bD,
\]  
where $\mathrm{T}_{[\bU]} \mathcal{M}: 
\mathcal{M} \to \R^{n\times r}$ denotes the tangent space to $\mathcal{M}$ at $[\bU]$.  According to \citet[Propositions 9.38 and 9.44]{boumal2023introduction}), the Riemannian gradient and Hessian at $[\bU]$ along a direction $\bm \xi \in \mathrm{T}_{[\bU]} \mathcal{M}$, denoted by $\mathrm{grad} F([\bU])[\bm \xi]$ and $\mathrm{Hess} F([\bU])[\bm \xi, \bm \xi]$, are computed by 
\begin{align}\label{eq:gd Hess}
    & \mathrm{grad} F([\bU])[\bm \xi] = \left\langle \nabla F(\bU), \mathrm{lift}_{\bU}(\bm \xi) \right\rangle, \notag\\
    & \mathrm{Hess} F([\bU])[\bm \xi, \bm \xi] = \left\langle \nabla^2 F(\bU)[\mathrm{lift}_{\bU}(\bm \xi)] , \mathrm{lift}_{\bU}(\bm \xi) \right\rangle.
\end{align}
\textbf{Geodesic strong convexity on $\mathcal{M}$.} Equipped with the above setup, we analyze the optimization landscape of Problem \eqref{eq:MC} around the global optimal solutions in the ReLU sampling regime. Although Problem \eqref{eq:MC} does not possess a benign {\em global} optimization landscape, we show that it has a favorable {\em local} optimization landscape. 
\begin{theorem}\label{thm:noiseless:2}
Under Assumptions \ref{ass:uistar} and \ref{ass:omegai}, $F(\cdot)$ is geodesically strongly convex on the quotient manifold $\mathcal{M}$ at $[\bU^\star]$, i.e., for all $\bm \xi \in \mathrm{T}_{[\bU^\star]} \mathcal{M}$, 
\begin{align}\label{eq:Hess}
    \mathrm{Hess} F([\bU^\star])[\bm\xi, \bm\xi] \geq \frac{\gamma}{2} \|\mathrm{lift}_{\bU^\star} (\bm\xi)\|_F^2,
\end{align}
where 
\begin{align}\label{eq:gamma}
        \gamma := \min_{\bD \in \mathcal{S}^\perp_{\bU^\star}, \|\bD\|_F = 1} \|(\bU^\star \bD^T + \bD\bU^{\star T})_\Omega\|_F^2 > 0. 
\end{align}
\end{theorem}

We defer the proof of this theorem to \Cref{pf:thm 2}. Intuitively, this theorem demonstrates that the objective function $F(\cdot)$ is strongly convex on the manifold $\cal M$ in the neighborhood of $\bm U^\star$. Consequently, if a tailored initialization in the local neighborhood of $\bm U^\star$ is available, GD is guaranteed to find a global optimal solution.
\vspace{-0.1in}
\paragraph{From geodesic strong convexity on $\mathcal{M}$ to strong convexity in Euclidean space.} Notably, the above geodesic strong convexity on the manifold $\mathcal{M}$ implies strong convexity of $F(\cdot)$ along some directions in Euclidean space. Specifically, for any $\bU \in \mathbb{R}^{n \times r}
$, let $\bU^T \bU^\star = \bP\bm \Sigma \bQ^T$ be a singular value decomposition of $\bU^T \bU^\star$, where $\bm P, \bm Q \in \mathcal{O}^{r}$ and $\bm \Sigma \in \R^{r\times r}$, and $\widetilde{\bU} := \bU\bP\bQ^T$. Using this and \eqref{eq:Hess}, we can show 
\begin{align*} 
        \nabla^2 F(\bm U^\star)[\widetilde{\bU} - \bU^\star, \widetilde{\bU} - \bU^\star]  \geq \frac{\gamma}{2} \|\widetilde{\bU} - \bU^\star\|_F^2.
\end{align*}
Please refer to \cref{sec:app:noiseless} for detailed proof. 

\subsection{Noisy Case and General Deterministic Sampling}\label{subsec:noisy}

In this subsection, we extend our above results to the noisy case with a general deterministic sampling regime. Toward this goal, we need the following assumption on the sampling pattern, which generalizes Assumptions \ref{ass:uistar} and \ref{ass:omegai}. 

\begin{assumption}\label{ass:main}
    There exists a collection of index set $\{\mI_1, \mI_2, \dots, \mI_K\}$ such that the following conditions hold: \vspace{-0.1in}
    \begin{itemize}
        \item[(a)] $\bigcup_{k \in [K]} \mI_k = [n]$;  
        \item[(b)] $\mI_k \times \mI_k \subseteq \Omega$ holds for all $k \in [K]$;
        \item[(c)] There exists $\lambda > 0$ such that $\bU_k^{\star T} \bU_k^\star \succeq \lambda \bI_r$ for each $k \in [K]$, where $\bU_k^\star$ is the matrix with the rows consisting of $\bu_i^\star$ for all $i \in \mI_k$; 
        \item[(d)] For any $k \neq l \in [K]$, there exists a path $k_0 \rightarrow k_1 \rightarrow \cdots \rightarrow k_s$ such that $k_0 = k$, $k_s = l$ and $\mI_{k_{j-1}} \times \mI_{k_j} \subseteq \Omega$ for all $j \in [s]$. 
    \end{itemize}
\end{assumption}
Now, let us explain these conditions in more detail. Condition (a) guarantees that each row of $\bm U^\star$ belongs to at least one submatrix; Condition (b) guarantees that the diagonal blocks of $\bm M$ index by $\mathcal{I}_k$ for all $k \in [K]$ are fully observed; Condition (c) indicates that there are enough rows in each part such that the submatrix $\bU_k^\star$ is of full row rank; Finally, condition (d) ensures that there are some off-diagonal blocks that are also fully observed and any two of these blocks are connected via a path. We note that the last condition is strict but aids significantly in the proof of \cref{thm:noisy:1}. With a slightly more careful analysis, this condition could be weakened so that off-diagonal blocks are only partially observed. In \Cref{subsec:Gaus}, we show that under the setting where the entries of $\bU^\star$ are i.i.d. sampled from the standard Gaussian distribution, $r \leq O( \log n)$, and the noise is bounded, \Cref{ass:main} holds with high probability. 

\vspace{-0.1in}
\paragraph{Comparison to \citet{bishop2014deterministic}.} Assumption \ref{ass:main} is similar to the assumptions in \cite{bishop2014deterministic}, but as far as we know, neither implies the other. More specifically, \citet{bishop2014deterministic}  consider deterministic sampling and assume that a collection of subsets of $\Omega$ admit an ordering such that any two adjacent parts have enough overlap. However, in the ReLU sampling setting, this order is hard to construct, and it is not clear whether there exists such an ordered collection of subsets. Moreover, both the analysis and the algorithm proposed in \cite{bishop2014deterministic} heavily rely on these ordered subsets, while we only use our partition for analysis. 

\vspace{-0.1in}
\paragraph{General sampling regimes.} Notably, \Cref{ass:main} can be applied to more general sampling regimes. For example, consider a positive-threshold sampling regime, where the entry $m_{ij}$ of $\bm M$ is observed if $m_{ij} \ge \eta$ for a constant $\eta \geq 0$. In particular, ReLU sampling corresponds to the case $\eta = 0$.

\begin{figure}[ht]
\begin{center}
\centerline{\includegraphics[width=0.5\columnwidth]{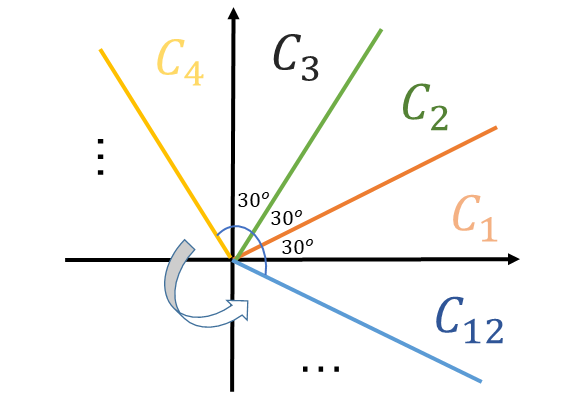}}\vspace{-0.1in}
\caption{{\bf A figure on the partition of the plane into 12 equal sectors.}}\label{figure:sector}
\end{center}
\vskip -0.2in
\end{figure}

\begin{example}[Positive-Threshold Mask]
Consider the noiseless case with $r = 2$ and the entry $m_{ij}$ is observed if $m_{ij} \ge \eta$ for a threshold $\eta >0$. We partition the 2D plane into 12 equal sectors, shown in Figure \ref{figure:sector}. Let us define a collection of index sets as follows:
\begin{align*}
    \mI_k := \{i \in [n] : \bu^\star_i \in \mathcal{C}_k\},\ \forall k \in [12].
\end{align*}
Obviously, we have $[n] = \cup_{k \in [K]} \mI_k$, and thus condition (a) in \Cref{ass:main} holds. For any $i, j \in \mathcal{I}_k$, the rows $\bu^\star_i$ and $\bu^\star_j$ belong to the same sector $\mathcal{C}_k$, which implies that the angle between $\bu^\star_i$ and $\bu^\star_j$ is less than ${\pi}/{6}$. Therefore, we have $m_{ij} = \bu^{\star T}_i \bu_j^{\star} = \|\bu^\star_i\|\|\bu^{\star}_j\| \cos({\pi}/{6}) = \sqrt{3}\|\bu^\star_i\|\|\bu^{\star}_j\|/2.$
As long as the threshold $\eta \leq \min\{\sqrt{3}\|\bu^\star_i\|\|\bu^{\star}_j\|/2:i, j \in \mathcal{C}_k,\ \forall k \in [12]\} $, the entry $m_{ij}$ can be observed for all  $i, j \in \mathcal{I}_k$, and thus $\mI_k \times \mI_k \subseteq \Omega$. This implies that condition (b) in \Cref{ass:main} holds. Condition (c) in  \Cref{ass:main} holds as long as each $\mathcal{I}_k$ contains at least two elements $i \neq j$ such that $\bm u_i^\star$ is not parallel to $\bm u_j^\star$. Finally, for any $i\in \mathcal{I}_k$ and $j \in \mathcal{I}_{k+1}$, one can verify that $m_{ij} \geq  \|\bu^\star_i\|\|\bu^{\star}_j\|/2$ using the similar argument.  
Consequently, as long as the threshold $\eta \leq \min\{\|\bu^\star_i\|\|\bu^{\star}_j\|/2: i \in \mathcal{I}_k,\ j \in \mathcal{I}_{k+1},\ \forall k \in [11]\} $, we have $\mathcal{I}_{k} \times \mathcal{I}_{k+1} \subseteq \Omega$. That is, all the near-diagonal blocks of $M$ will be fully observed, and thus condition (d) in \Cref{ass:main} holds.
\end{example}

\vspace{-0.1in}
\paragraph{Global optimality and local landscape analysis.} Based on \Cref{ass:main}, we are ready to characterize the global optimality and local optimization landscape of Problem \eqref{eq:MC} in the noisy case. 

\begin{theorem}\label{thm:noisy:1}
Let $\bm U \in \R^{n\times r}$ be any global optimal solution to Problem \eqref{eq:MC}. Under Assumption \ref{ass:main}, we have 
\begin{align*}
    \|\bm U\bm U^{T} - \bU^\star\bU^{\star T}\|_F \leq \frac{c}{\lambda} \|\bDelta\|_F
\end{align*} 
where $c >0$ depends on $\lambda$, $\|\bm U^\star\|_F$ and $\|\bm \Delta\|_F$. 
\end{theorem}
We defer the proof of the theorem to \Cref{pf:thm 3}. In particular, the dependence of $c$ on $\lambda$, $\|\bm U^\star\|_F$ and $\|\bm \Delta\|_F$ in specified in \eqref{eq:c depe}. Notably, this theorem generalizes \Cref{thm:noiseless:1} to the noisy case with general deterministic sampling satisfying \Cref{ass:main}.


\begin{theorem}\label{thm:noisy:2}
Let $\bm U \in \R^{n\times r}$ be any global optimal solution to Problem \eqref{eq:MC}. Suppose that the Assumption \ref{ass:main} holds and the noise matrix $\bm \Delta$ in \eqref{model:UV} satisfies 
\begin{align}\label{ineq:noise-upperbound}
\begin{split}
    \|\bDelta\|_F \leq \min & \left\{ \frac{\gamma}{8},\ \frac{\lambda\sqrt{\lambda}}{4(\sqrt{\lambda} + \sqrt{\gamma} + \|\bU^\star\|)},\ \frac{\lambda \sqrt{\gamma}}{4 \tilde{c}_0 \left(1 + \kappa(\sqrt{\gamma} + 2\|\bU^\star\|)\right)}\right\}
\end{split}   
\end{align} 
where $\gamma$ is provided in \eqref{eq:gamma} in \Cref{thm:noiseless:2}, $\tilde{c}_0 > 0$ depends on $\gamma, \bU^\star,$ and $\lambda$,  and $\kappa >0$ only depends on $\bU^\star$. Then, $F(\cdot)$ is geodesically strongly convex on $\mathcal{M}$ at $[\bm U]$, i.e., for all $\bm \xi \in \mathrm{T}_{[\bm U]} \mathcal{M}$,
\begin{align*}
    \mathrm{Hess} F([\bm U])[\bm \xi, \bm \xi] 
    \geq \left( \frac{\gamma}{8} - \|\bDelta\|_F \right)\|\mathrm{lift}_{\bm U}(\bm \xi)\|_F^2.
\end{align*}
\end{theorem}

We defer the proof of this theorem to  \Cref{pf:thm 4}. 
Intuitively, this theorem demonstrates that the objective function $F(\cdot)$ is strongly convex on the quotient manifold $\cal M$ in the neighborhood of $[\bm U]$.

\subsection{Verification of Assumptions}\label{subsec:Gaus}

In this subsection, we show that when the entries of $\bm U^\star$ are i.i.d. sampled from the standard Gaussian distribution, Assumptions \ref{ass:uistar}, \ref{ass:omegai}, and \ref{ass:main} all hold with high probability using a concentration argument. 

\vspace{-0.1in}
\paragraph{Noiseless setting.} We first show that when $\bm \Delta = \bm 0$ in \eqref{model:UV}, if the entries of $\bm U^\star$ are i.i.d. Gaussian random variables and the rank is small, Assumptions \ref{ass:uistar} and \ref{ass:omegai} hold with a high probability. 

\begin{proposition}\label{prop:gaussian-noiseless}
    Suppose that $r \leq  (\log_2 n)/2$ and the entries of $\bU^\star$ are i.i.d. sampled from the standard Gaussian distribution. Then, Assumptions \ref{ass:uistar} and \ref{ass:omegai} hold with probability at least $1 - \sqrt{n} \exp(- \sqrt{n}/{16})$.
\end{proposition}

We defer the proof of this proposition to \Cref{pf:as 1}. While the rank assumption is strict, we empirically observe in  \Cref{subsec:expranks} that it can be relaxed. This rank requirement arises from our proof technique. It is of great interest to find an alternative approach to improve our analysis technique.

\vspace{-0.1in}
\paragraph{Noisy setting.} For the noisy case, i.e., $\bm \Delta \neq  \bm 0$ in \eqref{model:UV}, we construct the collection of index sets mentioned in Assumption \ref{ass:main} in a similar manner to the partition shown in \Cref{figure:sector}. However, it is more difficult to characterize it in high dimensions. To address this, we introduce the concept of $\epsilon$-net (see Definition \ref{defi:eps-net}) to construct these index sets.
Under the Gaussian distribution and ReLU sampling, we prove the following proposition that shows as long as $r \le O(\log n)$ and the noise is small enough, Assumption \ref{ass:main} hold with high probability. We defer the proof to \Cref{pf:as 2}.

\begin{proposition}\label{prop:guassian}
Suppose that  $r\leq \log n/(4\log 3)$, the entries of $\bU^\star$ are i.i.d. sampled from the standard Gaussian distribution, and the noise matrix $\bDelta$ satisfies $\|\bDelta\|_\infty < \min\left\{ \|\bu_k^\star\|^2: k\in [n] \right\}/2$. Then, Assumption \ref{ass:main} holds with probability at least $1 - 1/n$. 
\end{proposition}


\begin{figure*}[t]
\begin{center}
	\begin{subfigure}{0.32\textwidth}
    	\includegraphics[width = 1\linewidth]{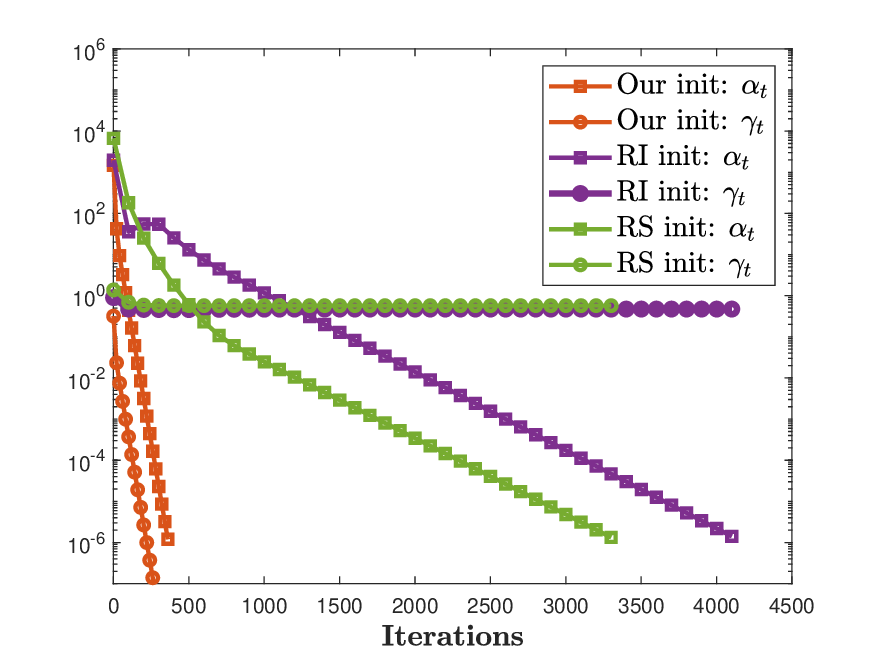}\vspace{-0.05in}
    \caption{Noiseless case: $\sigma = 0$} 
    \end{subfigure} 
    \begin{subfigure}{0.32\textwidth}
    	\includegraphics[width = 1\linewidth]{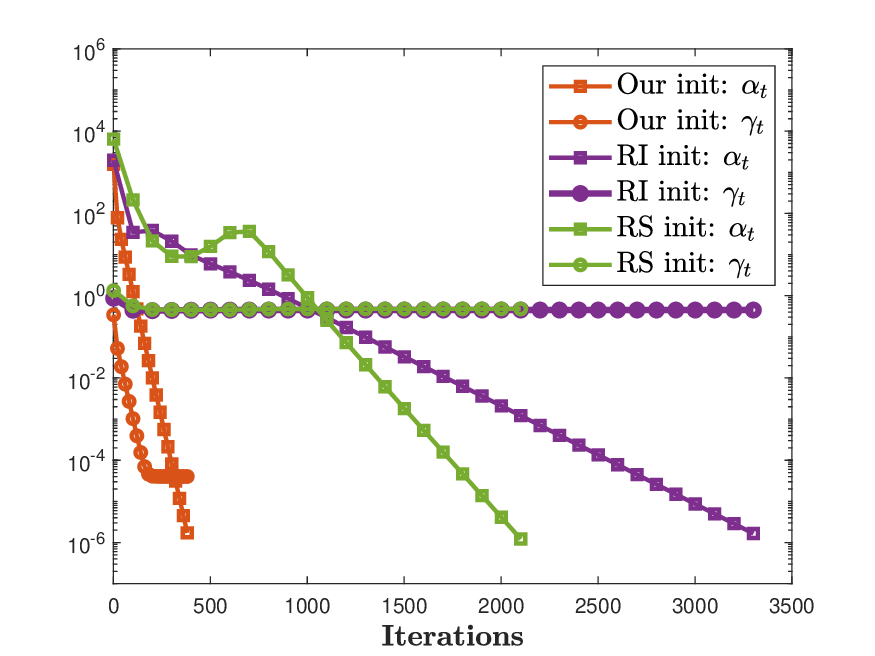}\vspace{-0.05in}
    \caption{Noisy case: $\sigma = 10^{-4}$} 
    \end{subfigure}
    \begin{subfigure}{0.32\textwidth}
    	\includegraphics[width = 1\linewidth]{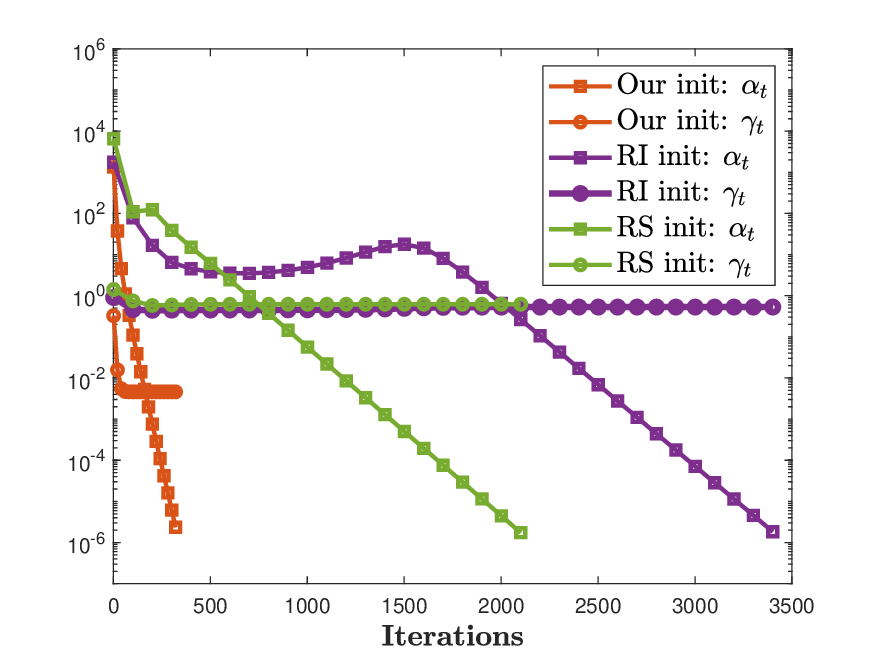}\vspace{-0.05in}
    \caption{Noisy case: $\sigma = 10^{-2}$} 
    \end{subfigure} 
    \vspace{-0.1in}
    \caption{\textbf{Convergence and recovery performance of GD for MC with ReLU sampling under different initialization schemes.} 
    }
    \vspace{-0.15in} 
    \label{fig:init}
\end{center}
\end{figure*}

\section{Experimental Results}\label{sec:expe}

In this section, we conduct numerical experiments to validate our theoretical developments and demonstrate the performance of several state-of-the-art algorithms on the MC problem with ReLU sampling. 
All of our experiments are implemented in MATLAB R2023a on a PC with 32GM memory and Intel(R) Core(TM) i7-11800H 2.3GHz CPU. 
We first introduce our tailor-designed initialization for GD to solve Problem \eqref{eq:MC} in \Cref{subsec:init}. Then, we study the convergence and recovery performance of GD with the proposed initialization and other initialization in \Cref{subsec:conv} to verify the developed theorems. We compare our proposed approach with some existing approaches in \Cref{subsec:comp}. Moreover, we investigate the effect of rank on the accuracy of matrix completion in \Cref{subsec:expranks}. Finally, we conduct experiments with Euclidean distance matrix completion in \cref{sec:eucliddist}.

\begin{table*}[t]
\caption{Comparison of the norm of gradient, function value, and completion error returned by GD with different initialization.}\label{table:1}
\label{sample-table}
\vskip -0.2in
\begin{center}
\begin{tabular}{lcccr}
\toprule
$\sigma = 0$ & $\alpha_{T}$ & $\beta_{T}$ & $\gamma_{T}$ \\
\midrule
{\bf Our init}  & $(9.7 \pm 0.16)\cdot 10^{-7}$ & $(3.3 \pm 0.36) \cdot 10^{-14}$ &  $ (7.6 \pm 0.7) \cdot 10^{-11}$ \\
\text{RI init} & $(6.8 \pm 0.3)\cdot 10^{-3}$ & $ (7.9 \pm 2.1) \cdot 10^{3}$ & $0.5 \pm 0.13$ \\ 
\text{RS init} & $0.1 \pm 0.45$ & $ (7.1 \pm 3.8) \cdot 10^{3} $ & $0.45 \pm 0.24$ \\
\midrule\midrule
$\sigma = 10^{-4}$ & $\alpha_{T}$ & $\beta_{T}$ & $\gamma_{T}$ \\
\midrule
{\bf Our init}    & $(9.7 \pm 0.14)\cdot 10^{-7}$ & $(1.8 \pm 0.03)\cdot 10^{-4}$ & $(4.4 \pm 0.13)\cdot 10^{-5}$ \\
{RI init} & $(0.39 \pm 1.2)\cdot 10^{-5}$ & $(8.4 \pm 2.1) \cdot 10^3$  & $0.51 \pm 0.13$ \\
\text{RS init} & $(1.5 \pm 6.5)\cdot 10^{-3}$ & $(8.4 \pm 2.0) \cdot 10^3$ & $0.51 \pm 0.13$ \\ 
\bottomrule
\end{tabular}
\end{center}
\vskip -0.1in
\end{table*}

\subsection{Our Proposed Algorithm}\label{subsec:init}

A natural approach to solving Problem \eqref{eq:MC} is to apply GD as follows: Given the current iterate $\bm U^{(t)}$, we generate the next iterate via

\vspace{-3mm}
\begin{align}\label{eq:gd}
\bm U^{(t+1)} = \bm U^{(t)} - \eta  \nabla F (\bm U^{(t)}),\ t \ge 0,
\end{align}
where $\bm Z = (\bm U\bm U^T)_{\Omega} - \bm M_{\Omega}$ and $\nabla F(\bm U) = \left( \bm Z +  \bm Z^T \right)\bm U$.
According to \Cref{thm:noiseless:2}, Problem \eqref{eq:MC} with the ReLU sampling exhibits a benign {\em local} optimization landscape, even though it possesses a complicated {\em global} landscape as illustrated in \cref{fig:crit}. Consequently, GD may not be effective for solving Problem  \eqref{eq:MC} unless a proper initial point $\bm U^{(0)}$ is available. This motivates us to propose a tailor-designed initialization as follows: We first randomly generate a matrix $\bm Y \in \R^{n\times r}$, whose entries are i.i.d. sampled from the standard normal distribution. Then, we construct $\bm Q  = \bm Y \bm Y^T$ and generate a new matrix $\bar{\bm Q} \in \R^{n\times n}$ via
\begin{align}\label{eq:init}
    \bar{q}_{ij} = \begin{cases}
        m_{ij},&\ \text{if}\ (i,j) \in \Omega, \\
        -|q_{ij}|,&\ \text{otherwise}.
    \end{cases}
\end{align}
Finally, we apply a truncated SVD to $\bar{\bm Q}$ to obtain an initial point $\bm U^{(0)} \in \mathcal{O}^{n\times r}$, where the columns of $\bm U^{(0)}$ consist of right singular vectors of $\bar{\bm Q}$ corresponding to its top $r$ singular values. Notably, we leverage the low-rank structure of $\bm M$ and the ReLU sampling to design the initialization scheme. According to this and \eqref{eq:gd}, we now summarize our proposed method for solving Problem \eqref{eq:MC} in \Cref{alg:1}. 

\begin{algorithm}[!tb]
   \caption{GD for MC with ReLU sampling}
   \label{alg:1}
\begin{algorithmic}
   \STATE {\bfseries Input:} observed set $\Omega$, observed matrix $\bm M_{\Omega}$
   \STATE Generate $\bar{\bm Q} \in \R^{n\times n}$ using \eqref{eq:init}
   \STATE Apply a truncated SVD to $\bar{\bm Q}$ to obtain $\bm U^{(0)} \in \mathcal{O}^{n\times r}$ 
   \FOR{$t=0,1,2,\dots$}
   \STATE $\bm U^{(t+1)} = \bm U^{(t)} - \eta  \nabla F(\bm U^{(t)})$
   \ENDFOR
\end{algorithmic}
\end{algorithm}

\subsection{Convergence Behavior and Recovery Performance}\label{subsec:conv}

In this subsection, we study the convergence behavior and recovery performance of GD in \Cref{alg:1} for solving Problem \eqref{eq:MC} in the ReLU sampling regime. To measure the convergence and recovery performance of the studied algorithms, we employ the following metrics at the $t$-th iteration: 
\begin{align*} 
    & \text{Norm of gradient:}\ \alpha_t = \|\nabla F(\bm U^{(t)})\|_F, \\
    & \text{Function value:}\ \beta_t = F(\bm U^{(t)}), \\
    & \text{Completion error:}\ \gamma_t =\|\bm U^{(t)}\bm U^{(t)^T} - \bm M\|_F/{\|\bm M\|_F}. 
\end{align*}
Obviously, if $\alpha_t$ is smaller, then $\bm U^{(t)}$ will be closer to a critical point. Moreover, if $\beta_t$ and $\gamma_t$ are smaller, then $\bm U^{(t)}$ will be closer to ground truth. 

\vspace{-0.1in}
\paragraph{Our tailor-designed vs. random-imputation spectral initialization (RI) \& random spectral (RS) initialization.} To demonstrate the effectiveness of our tailor-designed initialization, we compare our proposed initialization with a random-imputation spectral (RI) and a random spectral (RS) initialization. Specifically, the RI initialization proceeds as follows: we generate a matrix $\bm Y \in \R^{n\times r}$, whose entries are i.i.d. sampled from the standard normal distribution. Then, we construct $\bm Q = \bm Y\bm Y^T$ and generate a matrix $\bar{\bm Q}$ via setting $\bar{q}_{ij}  = m_{ij}$ if $(i,j) \in \Omega$. Finally, we generate $\bm U^{(0)} \in \mathcal{O}^{n\times r}$ using the truncated SVD to $\bar{\bm Q}$ as explained in \Cref{subsec:init}. Moreover, the RS initialization proceeds exactly the same way as above, but without the imputation scheme $\bar{q}_{ij}  = m_{ij}$ if $(i,j) \in \Omega$. 

\vspace{-0.1in}
\paragraph{Experimental setup.} In our experiments, we set $n = 200$ and $r=5$. We generate data matrix $\bm M$ according to the model \eqref{model:UV} with different noise level $\sigma \in \{0,10^{-4},10^{-2}\}$ and sample the observed entries via ReLU sampling, e.g. \eqref{eq:Ob}. For each noise level, we generate $20$ data matrices and run GD with our proposed, RI, and RS initialization on each data matrix, respectively. In each test, we terminate the algorithm when the Frobenious norm of the gradient at the $T$-th iteration is less than $10^{-6}$ or $T \ge 5000$. 

\vspace{-0.1in}
\paragraph{Experimental results.} To demonstrate the convergence and recovery performance of our proposed approach, we calculate and report the mean and standard deviation of the norm of gradient $\alpha_T$, function value $\beta_T$, and completion error $\gamma_T$ averaged over 20 runs in \Cref{table:1}. An additional result for $\sigma = 10^{-2}$ can be found in \cref{sec:app:addexpt}. In addition, we select one run from the 20 runs and plot norms of gradients $\{\alpha_t\}_{t\ge 0}$ and completion errors $\{\gamma_t\}_{t\ge 0}$ against the number of iterations in \Cref{fig:init}. 
In the noiseless case, i.e., $\sigma=0$, it is observed that GD with the tailor-designed initialization efficiently finds the ground truth at a linear rate, while GD with the RI or RS initialization converges a critical point that is not globally optimal. In the noisy case, i.e., $\sigma=10^{-4}$, we observe that GD with the tailor-designed initialization converges to a critical point within an $O(\sigma)$-neighborhood of the ground truth, while GD with the RI or RS initialization converges a critical point that is significantly distant from the ground truth. 

The results in \Cref{table:1} and \Cref{fig:init} support our theoretical results in \Cref{thm:noiseless:1} and \Cref{thm:noiseless:2}. They empirically demonstrate that the optimal solutions to Problem \eqref{eq:MC} can recover the ground truth but this problem only has a local benign optimization landscape. 
Additionally, the comparison between our proposed initialization and the RI spectral initialization further highlights the effectiveness of our tailor-designed initialization.

\subsection{Comparison to Existing Methods}\label{subsec:comp}

\begin{figure}
\begin{center}
\centerline{\includegraphics[width = 0.8\linewidth]{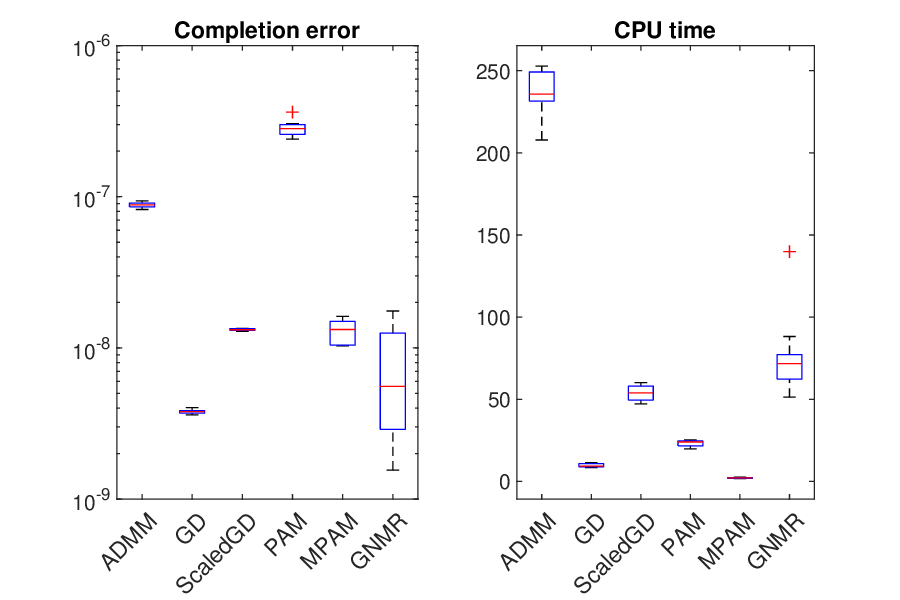}}\vspace{-0.15in}
\caption{Comparison of completion error and CPU time for a variety of MC algorithms in the noiseless case.}\label{fig:time}
\end{center}
\vspace{-0.3in}
\end{figure}

In this subsection, we compare our proposed method with some state-of-the-art methods for solving MC in the ReLU sampling regime in terms of recovery performance and computational efficiency. This includes the alternating direction method of multipliers (ADMM) \cite{boyd2011distributed} for the convex formulation of MC, scaled gradient descent (ScaledGD) for MC studied in \cite{tong2021accelerating}, proximal alternating minimization (PAM) for solving the formulation proposed in \cite{saul2022nonlinear}, momentum proximal alternating minimization (MPAM) for solving nonlinear matrix decomposition studied in \cite{seraghiti2023accelerated}, and Gaussian-Newton matrix recovery (GNMR) method for low-rank MC \cite{zilber2022gnmr}. To guarantee the performance of ScaleGD, we employ our tailored-designed initialization scheme in \Cref{subsec:init} to initialize it. We refer the reader to \Cref{subsec:effi} for the problem formulations and implementation details. 

\vspace{-0.1in}
\paragraph{Experimental setup.} In our experiments, we set $n=1000, r=20$ and generate data matrix $\bm M$ according to the model \eqref{model:UV} with different noise levels $\sigma \in \{0, 10^{-2}\}$. For each noise level, we generate $10$ data matrices and run the MC algorithms on each data matrix. In each test, we terminate our algorithm when the Frobenius norm of the gradient is less than $10^{-4}$. The termination criteria of other algorithms are specified in \Cref{subsec:effi}. 


\vspace{-0.1in}
\paragraph{Experimental results.} To compare the recovery performance and computational efficiency of the tested methods, we calculate and report the completion error and CPU times averaged over 10 runs using box plots in \Cref{fig:time}. See an additional figure for the noisy case in \cref{sec:app:addexpt}. It is observed that our proposed method can achieve a comparable recovery performance to the other tested methods in both noiseless and noisy cases. Moreover, our proposed method is as fast as MPAM, slightly faster than PAM and ScaledGD, and substantially faster than ADMM and GNMR in terms of computational efficiency. 

It is notable that all methods compared here have reasonable completion error on this ReLU MC problem. 
It was shown empirically in \cite{naik2022truncated} that other methods, including direct nuclear norm minimization, do not work as well. It is of great interest to develop a broader theory to understand or clarify this discrepancy among algorithms for the ReLU sampling and more general entry-dependent MC problems.

\vspace{-0.1in}
 \subsection{Completion with Various Ranks} \label{subsec:expranks}

In Figure~\ref{fig:rankvary}, we investigate completion accuracy in the noiseless setting when we increase the rank. In all cases, completion is successful well beyond our assumed rank bound of $r \leq \frac12 \log n$.  Consider $n=200$; even though $\log_2(200) \approx 7.6$, completion only begins to break down around $r=45$. 
This behavior was also reported in \cite{naik2022truncated}.
A deeper understanding of this fundamental limit is an exciting question for future work.

\begin{figure}[t]
\begin{center}
\includegraphics[width = .6\linewidth]{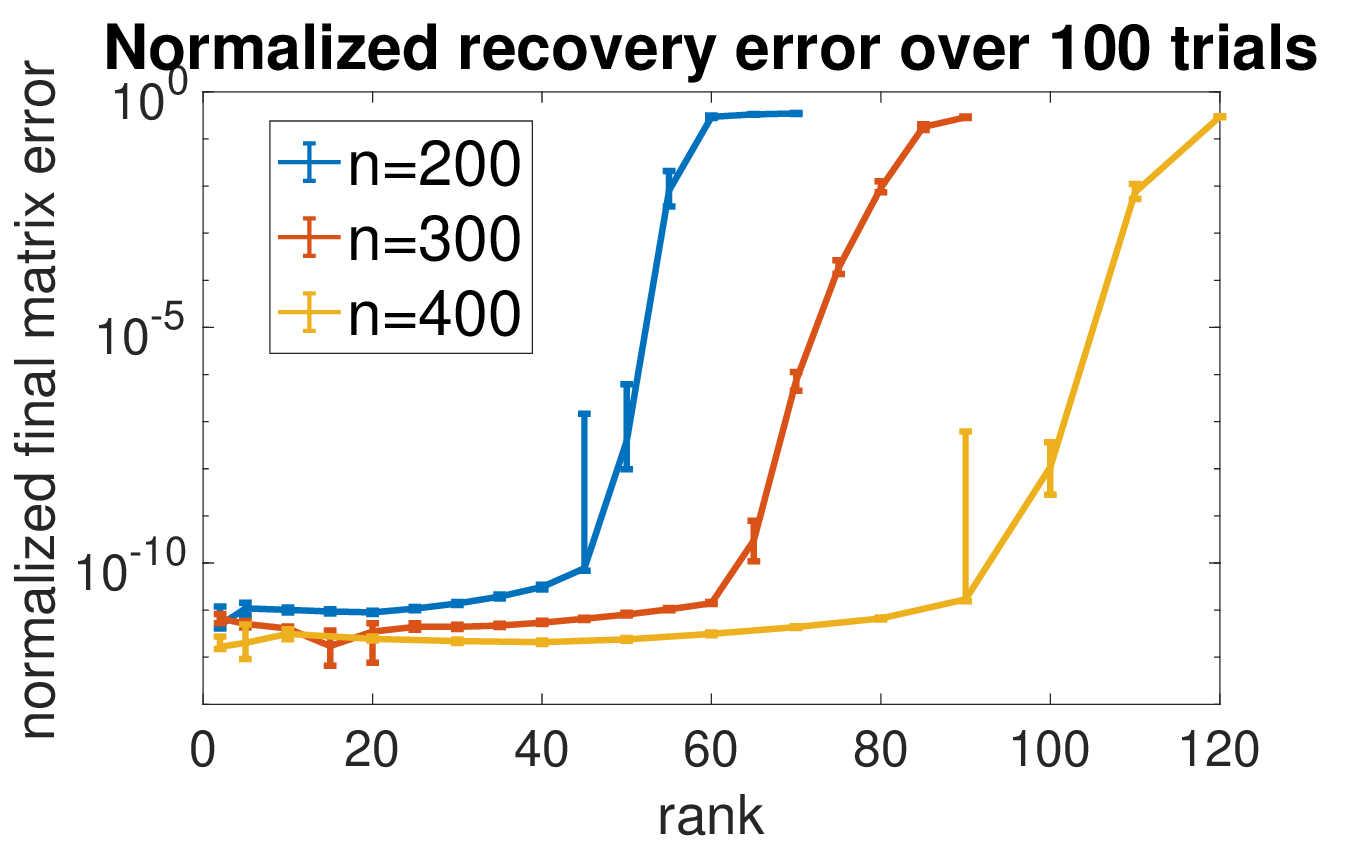}\vspace{-0.1in}
    \caption{Completion in the noiseless case as rank varies, with dimension $n=200, 300, 400$. Line shows the median error, and bars show the 25\% and 75\% error quantiles for 100 trials.} 
    \label{fig:rankvary}
\end{center}
\vspace{-0.25in}
\end{figure} 

\subsection{Euclidean Distance Matrix Completion}\label{sec:eucliddist}

Centered Euclidean distance matrices are low-rank and PSD. In applications, we may only have access to a small number of pairwise distances between points, and we would like to complete the Euclidean distance matrix to know all pairwise distances. 
In this section, we experiment with the scenario where we observe only the smallest distances in the matrix.
We generate Euclidean distance matrices synthetically, with dimension 200, and our observation is a fraction of the smallest entries. 
As shown in \cref{fig:euclid}, depending on rank and fraction of entries,  \cref{alg:1} recovers the matrix exactly.

Note that after centering, 50\% of entries are negative, but for a real Euclidean distance matrix none of the entries are negative. In a real data setting, we would want to design an algorithm to approximately center the matrix given only the observed entries.

\begin{figure}[t]
\begin{center}
\includegraphics[width = .6\linewidth]{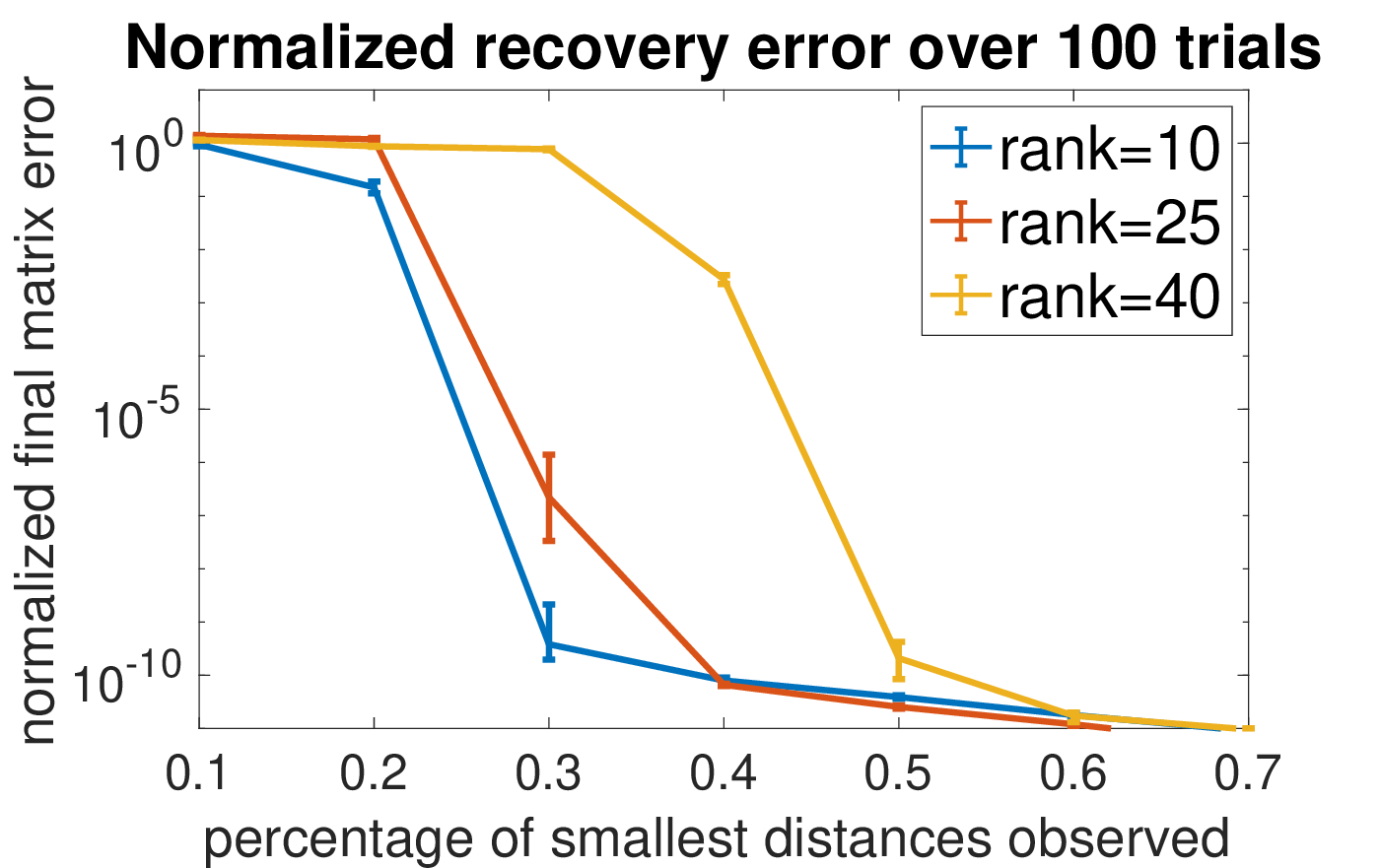}\vspace{-0.1in}
    \caption{Matrix completion for Euclidean distance matrices where only a fraction of the smallest entries are observed. Line shows the median error, and bars show the 25\% and 75\% error quantiles for 100 trials.} 
    \label{fig:euclid}
\end{center}
\vspace{-0.15in}
\end{figure}

\section{Conclusions}\label{sec:concl} 

In this work, we studied the MC problem with ReLU sampling. Under mild assumptions on the underlying matrix, we showed that global optimal solutions of the low-rank formulation of MC recover the underlying matrix in both noiseless and noisy cases. Moreover, we also characterized the local optimization landscape, which is strongly convex on the quotient manifold in the neighborhood of global optimal solutions. Finally, we proposed a tailor-designed initialization for GD to optimize the studied formulation and conducted extensive experiments to showcase the potential of our proposed approach. 

While our work fills a gap in the literature, there are still limitations that need to be addressed, starting with generalizing from SPSD matrices to rectangular matrices and considering unknown rank $r$. Weakening \cref{ass:main}(d) is of great interest, as mentioned before. 
Additionally, we point out that our method and initialization have prior knowledge that may not be available in general, e.g., our initialization explicitly uses the ReLU sampling assumption.
It would be interesting to empirically explore the sensitivity of MC algorithms to this coupling between sampling scheme and initialization.
A final general future direction of study is to extend our landscape analysis to more general deterministic and entry-dependent sampling regimes. 


\section*{Acknowledgements} The work of H.L. was supported in part by NSF China under Grant 12301403 and Young Elite Scientists Sponsorship Program by CAST 2023QNRC001. The work of L.B. and P.W. was supported in part by ARO YIP award W911NF1910027, NSF CAREER award CCF-1845076, and DoE award DE-SC0022186. QQ acknowledges support from NSF CAREER CCF-2143904, NSF CCF-2212066, and NSF CCF-2212326.

\bibliographystyle{abbrvnat}
\bibliography{matrix_completion}

\newpage
\appendix
\onecolumn
\begin{center}
{\Large \bf Supplementary Material}
\end{center}\vspace{-0.15in}
\par\noindent\rule{\textwidth}{1pt}

The appendix is organized as follows: \Cref{sec:appdix:proof} contains all the proofs for the main results. Specifically, in Section \ref{sec:app:noiseless}, we focus on the noiseless ReLU case and provide proofs for all the theorems presented in Section \ref{subsec:noiseless}. In Section \ref{sec:app:noisy:general}, we shift our attention to the noisy general case and provide proofs for the theorems found in Section \ref{subsec:noisy}. In Section \ref{sec:app:gaussian}, we concentrate on the Gaussian and ReLU sampling settings and present all the required proofs for Section \ref{subsec:Gaus}. \Cref{sec:supp:expt} provides detailed setups of our experiments and includes additional numerical results to supplement \Cref{sec:expe}.

\section{Proofs of Main Results}\label{sec:appdix:proof}

Before we proceed, let us recall some definitions and notions. Based on the block partition in \eqref{eq:Upartition},  we denote by $\bM^{(i,j)} \in \mathbb{R}^{n_i \times n_j}$ the $(i,j)$-th block of $\bM$, i.e., $\bM^{(i,j)} = \bU^{\star}_{i} \bU^{\star T}_{j}$ for any pair of indices $(i, j)\in [2^r] \times [2^r]$. Moreover, we denote the set of observations in the $(i,j)$ block by $\Omega_{i,j} = \{ (k, l) \in [n_i] \times [n_j]: m^{(i,j)}_{kl} \geq 0 \}$ 
where $m^{(i,j)}_{kl}$ is the $(k,l)$-th entry of $\bM^{(i,j)}$. Let a column vector $\bu^{\star}_{i, k} \in \R^r$ denote the $k$-th row of $\bU^\star_{i}$ for each $k \in [n_i]$. Then, we have $m^{(i,j)}_{kl} = \bu^{\star T}_{i, k}\bu_{j,l}^{\star}$. Since $\bu^{\star T}_{i,k}\bu^{\star}_{i,l} \geq 0$ always holds for all $k,l \in [n_i]$ due to the fact that the rows of $\bU^\star_i$ have the same sign, we have $\Omega_{i,i} = [n_i] \times [n_i]$ for all $i \in [2^r]$ under the ReLU sampling \eqref{eq:Ob}.

\subsection{Proofs for Section \ref{subsec:noiseless}: Noiseless Case and ReLU Sampling}
\label{sec:app:noiseless}

Firstly, we need to show that the horizontal space $\mathcal{S}^\perp_{\bU}$, which is the orthogonal complement of $\mathcal{S}_{\bU}$ defined in \eqref{defi:verticalspace}, is given by \eqref{defi:horizontalspace}. Indeed, for arbitrary $\bD \in \mathcal{S}^\perp_{\bU}$, due to the orthogonality, we obtain that $\langle \bD, \bU \bR \rangle = \langle \bU^{T} \bD, \bR \rangle = 0$ holds for all skew-symmetric matrix $\bR$. This holds if and only if $\bU^{T} \bD$ is symmetric. Consequently, $\mathcal{S}^\perp_{\bU}$ is given by
\begin{align}\label{eq:form-sperp}
    \mathcal{S}^\perp_{\bU} =  \left\{\bD \in \mathbb{R}^{n \times r} \mid \bU^{T} \bD = \bD^T \bU \right\}.
\end{align}

\subsubsection{Proof of Theorem \ref{thm:noiseless:1}}\label{pf:thm 1}

\begin{proof}[Proof of Theorem \ref{thm:noiseless:1}]
Let $\bm U \in \R^{n\times r}$ be an arbitrary optimal solution to Problem \eqref{eq:MC}. According to \eqref{model:UV} and $\bm \Delta = \bm 0$, we have $\bm M = \bm U^\star \bm U^{\star T}$. This, together with the optimality of Problem \eqref{eq:MC}, implies that  $(\bU \bU^T)_\Omega = (\bU^\star \bU^{\star T})_\Omega$. Let $\bm U = \left[\bU_1^T\ \bU_2^T\ \cdots\ \bU_{2^r}^T\right]^T$ be the same partition as that in \eqref{eq:Upartition}. Since all diagonal blocks are fully observed, we have $\bU_i \bU_i^T = \bU^\star_i \bU^{\star T}_i$. This, together with \Cref{ass:uistar} and \Cref{ineq:tech-lemma}, yields that there exists  $\bQ_i \in \mathcal{O}^r$ such that $\bU_i = \bU^\star_i \bQ_i$ for all $i \in [2^r]$.

Next, we show that the above $\bQ_i$ for all $i \in [2^r]$ are the same. For each tri-diagonal block, it follows from $(\bU \bU^T)_\Omega = (\bU^\star \bU^{\star T})_\Omega$ that $(\bU_{i+1} \bU_{i}^T)_{\Omega_{i+1, i}} = (\bU^\star_{i+1} \bU^{\star T}_i)_{\Omega_{i+1, i}}$ for all $i \in [2^r - 1]$. By substituting the equality $\bU_i = \bU^\star_i \bQ_i$ into it, we get $(\bU_{i+1}^\star \bQ_{i+1} \bQ_{i}^T \bU_{i}^{\star T})_{\Omega_{i+1, i}} = (\bU^\star_{i+1} \bU^{\star T}_i)_{\Omega_{i+1, i}}$. We write its entry-wise form as follows:  
\begin{align}\label{eq:qiqi1}
    \bm u^{\star T}_{i+1, k} \bQ_{i+1} \bQ_i^T \bm u^{\star}_{i, l}  = \bm u^{\star T}_{i+1, k} \bm u^{\star}_{i, l}, \quad \forall (k, l) \in \Omega_{i+1,i}.
\end{align}
According to Assumption \ref{ass:omegai}, $\{ \bm u^{\star}_{i+1, k} \bm u^{\star T}_{i, l} :(k, l) \in \Omega_{i+1,i} \}$ spans the whole $r$-by-$r$ matrix space, so \eqref{eq:qiqi1} holds if and only if $\bQ_{i+1} \bQ_i^T = \bm I_r$. Therefore, $\bQ_i = \bm Q_j$ for all $i \neq j$. This yields $\bU = \bU^\star \bQ$ with $\bQ \in \mathcal{O}^r$ for any optimal solution $\bU$. Then, we complete the proof.
\end{proof} 

\subsubsection{Proof of Theorem \ref{thm:noiseless:2}}\label{pf:thm 2}

Before we proceed, let us compute the gradient and Hessian of $F(\bm U)$. Obviously, the gradient is
\begin{align}\label{eq:grad}
    \nabla F(\bU) = (\bU\bU^T - \bM)_\Omega \bU.
\end{align}
Given a direction $\bD \in \mathbb{R}^{n\times r}$, we have 
    \begin{align*}
        \nabla F(\bU + t\bD) &= (\bU\bU^T + t\bU\bD^T + t\bD\bU^T + t^2\bD\bD^T - \bM)_\Omega (\bU + t\bD) \\
        &= \nabla F(\bU) + t(\bU\bD^T + \bD\bU^T)_\Omega \bU + t^2 (\bD\bD^T)_{\Omega}\bm U + t(\bU\bU^T - \bM)_\Omega \bD \\
        &\quad t^2(\bU\bD^T + \bD\bU^T)_{\Omega}\bm D + t^3 (\bD\bD^T)_{\Omega}\bm D.
    \end{align*}
Then, we compute its bilinear form of the Hessian along a direction $\bD \in \mathbb{R}^{n\times r}$ as follows:
    \begin{align}\label{eq:hess:direc}
    \begin{split}    
        \nabla^2 F(\bm U)[\bm D, \bm D] & =  \left\langle \bm D, \lim_{t \to 0}\frac{\nabla F(\bU + t\bD) - \nabla F(\bm U)}{t} \right\rangle  =  \langle \bD, (\bU\bU^T - \bM)_\Omega \bD + (\bU\bD^T + \bD\bU^T)_\Omega \bU \rangle \\
        & =  \langle \bD\bD^T, (\bU\bU^T - \bM)_\Omega \rangle + \langle \bD\bU^{T}, (\bU\bD^T + \bD\bU^{T})_\Omega  \rangle \\
        & = \langle \bD\bD^T, (\bU\bU^T - \bM)_\Omega \rangle + \frac{1}{2} \|(\bU \bD^T + \bD\bU^{T})_\Omega\|_F^2,
    \end{split}
    \end{align}
where the last inequality holds because $\langle \bD\bU^T, (\bU\bD^T + \bD\bU^T)_\Omega  \rangle = \langle  \bU\bD^T, (\bU\bD^T + \bD\bU^T)_\Omega  \rangle$ and
$$
\langle \bD\bU^T, (\bU\bD^T + \bD\bU^T)_\Omega  \rangle = \frac{1}{2}\langle \bU\bD^T + \bD\bU^T, (\bU\bD^T + \bD\bU^T)_\Omega  \rangle = \frac{1}{2} \langle (\bU\bD^T + \bD\bU^T)_\Omega, (\bU\bD^T + \bD\bU^T)_\Omega  \rangle.
$$

\begin{proof}[Proof of Theorem \ref{thm:noiseless:2}]
Using \eqref{eq:hess:direc} and $\bM = \bU^\star \bU^{\star T}$,  we have 
\begin{align}\label{eq:hessustar}
        \nabla^2 F(\bU^\star)[\bm D, \bm D] = \frac{1}{2} \|(\bU^\star \bD^T + \bD\bU^{\star T})_\Omega\|_F^2 \ge 0.
    \end{align}
This implies that the Hessian of $F$ at $\bU^\star$ is semi-definite positive. 

Next, we consider these directions $\bD$ such that $ \nabla^2 F(\bm U^\star)[\bm D, \bm D] = 0$. We aim to show 
\begin{align}\label{eq1:thm 2}
    \bD = \bU^\star \bV,\ \text{where}\ \bm V + \bm V^T = \bm 0.
\end{align}
First, for each $i \in [2^r]$, since the $i$-th diagonal part is fully observed, i.e., $\Omega_{i,i} = [n_i] \times [n_i]$, we have
\begin{align}\label{eq:uddu}
    \bU^\star_i \bD^T_i + \bD_i\bU^{\star T}_i = \bm 0.
\end{align}
It follows from \Cref{ass:uistar} that $\mathrm{rank}(\bm U_i^\star) = r$, and thus its pseudo-inverse is $\bm U_i^{\star\dagger} = (\bm U_i^{\star T}\bm U_i^{\star})^{-1}\bm U_i^{\star T} \in \R^{r\times n_i}$. 
By multiplying $\bU^{\star^\dag T}_i$ on the right hand side of \eqref{eq:uddu}, we get $\bD_i = - \bU^\star_i \bm D_i^T\bU^{\star^\dag T}_i = \bU^\star_i \bV_i$, where $\bV_i := -\bm D_i^T\bU^{\star^\dag T}_i$. 
Moreover, by multiplying $\bU_i^{\star^\dag}$ and $\bU_i^{\star^\dag T}$ on the left and right sides of \eqref{eq:uddu}, 
we get $\bD_i^T\bU_i^{\star^\dag T} + \bU_i^{\star^\dag}\bD_i = \bm 0$. This implies $\bV_i + \bV_i^T = \bm 0$, which means $\bV_i$ is a skew-symmetric matrix. Now, let us consider the $(i+1, i)$-th block of $(\bU^\star \bD^T + \bD\bU^{\star T})_\Omega$. According to \eqref{eq:uddu} and $\bD_i = \bU^\star_i \bV_i$ for all $i\in [2^r]$, we have
\begin{align*}
    \bm 0 = \left( \bU^\star_{i+1} \bD^T_i + \bD_{i+1}\bU^{\star T}_i \right)_{\Omega_{i+1, i}} =  \left( \bU^\star_{i+1} \bV^T_i \bU^{\star T}_i + \bU^\star_{i+1} \bV_{i+1} \bU^{\star T}_i \right)_{\Omega_{i+1, i}} = \left( \bU^\star_{i+1} \left(\bV^T_i +  \bV_{i+1} \right) \bU^{\star T}_i \right)_{\Omega_{i+1, i}}.
\end{align*}
Using the same argument in \eqref{eq:qiqi1}, we have $\bV^T_i +  \bV_{i+1} = \bm 0$. Thus, we have $\bV_{i} = \bV_{i+1}$ for all $i\in [2^r-1]$, and thus  $\bV_i$ for all $i \in [2^r]$ are the same. Consequently, the space spanned by all the directions $\bD$ satisfying $\nabla F(\bm U^\star)[\bm D, \bm D] = \bm 0$ is exactly the vertical space $\mathcal{S}_{\bU^\star}$ defined in \eqref{defi:verticalspace}, i.e., 
\begin{align*}
    \mathcal{S}_{\bU^\star} = \left\{\bD \in \R^{n\times r}: \nabla F(\bm U^\star)[\bm D, \bm D] = 0 \right\}.
\end{align*}
Define the constant 
\begin{align}\label{defi:gamma}
    \gamma := \min_{\bD \in \mathcal{S}^\perp_{\bU^\star}, \|\bD\|_F = 1} \|(\bU^\star \bD^T + \bD\bU^{\star T})_\Omega\|_F^2 > 0
\end{align}
We claim that $\gamma > 0$. Now, we prove this claim by contradiction. Indeed, suppose that $\gamma = 0$. This implies that there exists a $\bD \in \mathcal{S}^\perp_{\bU^\star} $ with $\|\bD\|_F = 1$ such that $\|(\bU^\star \bD^T + \bD\bU^{\star T})_\Omega\|_F^2 = 0$. This, together with \eqref{eq1:thm 2}, yields $\bD \in \mathcal{S}_{\bU^\star}$ and thus $\bD \in \mathcal{S}_{\bU^\star}\cap \mathcal{S}_{\bU^{\star}}^\perp = \{\bm 0\}$, which contradicts the fact that $\|\bD\|_F = 1$. Then, for any $\bD \in \mathcal{S}^\perp_{\bU^\star}$, \eqref{eq:hessustar} shows that  
\begin{align}\label{ineq:hesslowerbound}
        \nabla^2F(\bm U^\star)[\bm D, \bm D] = \frac{1}{2} \|(\bU^\star \bD_{\mathcal{S}_{\bU^\star}^\perp}^T + \bD_{\mathcal{S}_{\bU^\star}^\perp}\bU^{\star T})_\Omega\|_F^2 \geq \frac{\gamma}{2} \|\bD_{\mathcal{S}_{\bU^\star}^\perp}\|_F^2 = \frac{\gamma}{2} \|\bD\|_F^2.
    \end{align}
where $\bD_{\mathcal{S}_{\bU^\star}^\perp} = \mathrm{Proj}_{\mathcal{S}_{\bU^\star}^\perp}(\bm D)$ is the projection of $\bm D$ on $\mathcal{S}_{\bU^\star}^\perp$. Finally, substituting $\bm D = \mathrm{lift}_{\bU^\star} (\bm \xi) \in \mathcal{S}^\perp_{\bU^\star}$ into the above inequality for any $\bm \xi \in \mathrm{T}_{[\bU^\star]}\mathcal{M} $, together with \eqref{eq:gd Hess}, yields the desired result.
\end{proof}

\paragraph{From the geodesic strong convexity on $\mathcal{M}$ to the strong convexity in Euclidean space.} For any $\bU, \bV \in \mathbb{R}^{n \times r}
$, let $\bU^T \bU^\star = \bP\bm \Sigma \bQ^T$ be an SVD of $\bU^T \bU^\star$, where $\bm P, \bm Q \in \mathcal{O}^{r}$ and $\bm \Sigma \in \R^{r\times r}$. Moreover, let $\widetilde{\bU} := \bU\bP\bQ^T$ and $\bD := \widetilde{\bU} - \bU^\star$, then we must have $\bD \in \mathcal{S}_{\bU^\star}^\perp$. This is because, for any $\bE \in \mathcal{S}_{\bU^\star} = \{\bU^\star \bV: \bV + \bV^T = 0\}$, we have 
\begin{align}\label{eq:eutus}
    \langle \bE, \widetilde{\bU} - \bU^\star\rangle & = \langle \bU^\star \bV, \bU\bP\bQ^T - \bU^\star\rangle  = \langle \bV, \bU^{\star T} \bU\bP\bQ^T -  \bU^{\star T} \bU^\star\rangle \notag \\
    & = \langle \bV, \bQ \bm\Sigma\bQ^T -  \bU^{\star T} \bU^\star\rangle = \bm 0,
\end{align}
where the last equality holds because $\bV$ is a skew-symmetric matrix and $\bQ \bm \Sigma\bQ^T -  \bU^{\star T} \bU^\star$ is symmetric. Therefore, \eqref{ineq:hesslowerbound} implies that 
\begin{align}\label{ineq:strongconvex-euclidean}
        \nabla^2F(\bm U^\star)[\widetilde{\bU} - \bU^\star, \widetilde{\bU} - \bU^\star]  \geq \frac{\gamma}{2} \|\widetilde{\bU} - \bU^\star\|_F^2.
    \end{align}

\subsubsection{Proof of Lemma \ref{ineq:tech-lemma}}\label{sec:app:techlem} 

\begin{proof}[Proof of \Cref{ineq:tech-lemma}]
According to $r = \text{rank}(\bm U)$, we obtain that the pseudo-inverse of $\bU$ is $\bU^{\dagger} := (\bU^T\bU)^{-1}\bU^T \in \R^{r\times n}$. Then, one can verify $\|\bU^\dagger\| \leq 1/\sigma_{\min}(\bm U)$. Next, let $\bm \Psi := \bU\bU^T - \bV\bV^T$. Then, we compute
$$ 
\bm \Psi \bU^{\dagger T} = \bU - \bV\bV^T \bU^{\dagger T}, \quad \bU^{\dagger } \bm \Psi \bU^{\dagger T} = \bI_r - \bU^{\dagger} \bV\bV^T\bU^{\dagger T}. 
$$ 
Note that
\begin{align}\label{eq1:ineq}
    \|\bU^{\dagger } \bm \Psi \bU^{\dagger T}\|_F \le \| \bU^\dagger\|^2\|\bm \Psi\|_F \le \frac{\|\bm \Psi\|_F}{\sigma_{\min}^2(\bm U)}. 
\end{align}
Let $\bm V^T \bm U^{\dagger T} = \bm P\bm \Sigma \bm H^T$ be a singular value decomposition (SVD) of $\bV^T \bm U^{\dagger T} \in \R^{r\times r}$, where $\bm P, \bm H \in \mathcal{O}^r$ and $\bm \Sigma \in \R^{r\times r}$. Then, we have  
$$ 
\|\bI_r - \bm \Sigma^2\|_F = \|\bm I_r - \bm H \bm \Sigma^2 \bm H^T \|_F = \|\bI_r - \bU^{\dagger} \bV\bV^T\bU^{\dagger T}\|_F = \|\bU^{\dagger } \bm \Psi \bU^{\dagger T}\|_F  \leq \frac{\|\bm \Psi\|_F}{\sigma_{\min}^2(\bm U)}, $$ 
where the inequality follows from \eqref{eq1:ineq}. This implies 
\begin{equation}\label{eqn:bnd4SmI}
\|\bm \Sigma - \bI_r\|_F = \|(\bm \Sigma + \bI_r)^{-1}(\bm \Sigma^2 - \bI_r)\|_F \leq \|\bm \Sigma^2 - \bI_r \|_F \leq \frac{\|\bm \Psi\|_F}{\sigma_{\min}^2(\bm U)}.
\end{equation}
Finally, by choosing $\bQ = \bm P\bm H^T$, we get 
$$
\bm \Psi \bU^{\dagger T} = \bU - \bV\bV^T\bU^{\dagger T} = \bU - \bV\bQ + \bV \bm P(\bI_r - \bm \Sigma)\bm H^T.
$$
Therefore, we obtain 
$$ \|\bU - \bV\bQ\|_F \leq \|\bm \Psi \bU^{\dagger T}\|_F + \|\bV \bm P(\bm I_r - \bm \Sigma) \bm H^T\|_F \leq \frac{\|\bm \Psi\|_F}{\sigma_{\min}(\bm U)} + \frac{\|\bV\| \|\bm \Psi\|_F}{\sigma_{\min}^2(\bm U)}, $$
where the last inequality follows from  $\|\bU^\dagger\| \leq 1/\sigma_{\min}(\bm U)$ and \eqref{eqn:bnd4SmI}. 
\end{proof}

Now, we remove the full rank assumption $\text{rank}(\bm U) = r$ and prove the following more general lemma.

\begin{lemma}\label{lem:UV}
For arbitrary $\bU, \bV \in \R^{n\times r}$ with $r \le n$,  there exists an orthogonal matrix $\bQ \in \mathcal{O}^{r}$ such that 
\begin{align}\label{eq:lem UV}
    \|\bU - \bV\bQ\|_F \leq \max\left\{ \frac{\sigma_{\min}^2(\bm U) + \|\bm V\|}{\sigma_{\min}^2(\bm U)}, \frac{1}{\sigma_{\min}(\bm V) } \right\}\|\bU\bU^T - \bV\bV^T\|_F,
\end{align}
where $\sigma_{\min}(\bm U)$ and $\sigma_{\min}(\bm V)$ denote the smallest non-zero singular value of $\bU$ and $\bV$, respectively.  
\end{lemma}

\begin{proof}
    For ease of exposition, we denote the rank of $\bm U$ by $\hat{r} = \text{rank}(\bU)$, where $\hat{r} \le r$. Let $\bU = \bm P_1 \bm \Sigma_1 \bm Q_1^T$ be a compact SVD of $\bU$, where $\bm P_1 \in \mathcal{O}^{n \times \hat{r}}$, $\bm \Sigma_1 \in \mathbb{R}^{\hat{r} \times \hat{r}}$, and $\bm Q_1 \in \mathcal{O}^{\hat{r}}$. Notice that 
    \begin{align*}
        \|\bU\bU^T - \bV\bV^T\|_F^2 & = \|\bm P_1\bm P_1^T(\bU\bU^T - \bV\bV^T) + (\bm I - \bm P_1\bm P_1^T)(\bU\bU^T - \bV\bV^T)\|_F^2 \\
        & = \|\bm P_1\bm P_1^T(\bU\bU^T - \bV\bV^T)\|_F^2 + \|(\bm I - \bm P_1\bm P_1^T)(\bU\bU^T - \bV\bV^T)\|_F^2 \\
        & = \|\bm P_1^T(\bU\bU^T - \bV\bV^T)\|_F^2 + \|(\bm I - \bm P_1\bm P_1^T)\bV\bV^T\|_F^2,
    \end{align*}
where the last equality holds because $\bm P_1^T \bm P_1 = \bI$ and $(\bm I - \bm P_1\bm P_1^T)\bU = \bm 0$. Using the same argument, we have
\begin{align*}
    \|\bm P_1^T(\bU\bU^T - \bV\bV^T)\|_F^2 & = \|\bm P_1^T(\bU\bU^T - \bV\bV^T)\bm P_1\bm P_1^T\|_F^2 + \|\bm P_1^T(\bU\bU^T - \bV\bV^T)(\bm I - \bm P_1\bm P_1^T)\|_F^2 \\
    & = \|\bm P_1^T(\bU\bU^T - \bV\bV^T)\bm P_1\|_F^2 + \|\bm P_1^T \bV\bV^T(\bm I -\bm P_1 \bm P_1^T)\|_F^2\\
    & \geq  \|\bm P_1^T(\bU\bU^T - \bV\bV^T)\bm P_1\|_F^2 = \|\bm \Sigma_1^2 - \bm P_1^T\bV\bV^T\bm P_1\|_F^2,
\end{align*}
where the last equality holds because $\bU = \bm P_1\bm \Sigma_1 \bm Q_1^T$. Putting the above results together yields 
\begin{align}\label{eq3:lem UV}
    \|\bU\bU^T - \bV\bV^T\|_F^2 \geq \|\bm \Sigma_1^2 - \bm P_1^T\bV\bV^T\bm P_1\|_F^2  + \|(\bm I - \bm P_1\bm P_1^T)\bV\bV^T\|_F^2.
\end{align}
Similarly, we also have
\begin{align}\label{eq1:lem UV}
 \|\bU - \bV\bQ\|_F^2 = \|\bm P_1^T(\bU - \bV\bQ)\|_F^2 + \|(\bm I - \bm P_1\bm P_1^T)\bV\bQ\|_F^2.   
\end{align}
Let $\bm P_1^T\bV = \bm P_2\bm \Sigma_2\bm Q_2^T$ be a compact SVD of $\bm P_1^T\bm V$, where $\bm P_2 \in \mathcal{O}^{n\times \hat{r}}$, $\bm \Sigma_2 \in \R^{\hat{r}\times \hat{r}}$, and $\bm Q_2 \in \mathcal{O}^{\hat{r}}$. By choosing $\bQ$ such that 
$\bm Q_2^T \bQ = \bm Q_1^T$, we have
\begin{align}\label{eq2:lem UV}
   \|\bm P_1^T(\bU - \bV\bQ)\|_F & = \|\bm \Sigma_1 \bm Q_1^T - \bm P_1^T\bV\bQ\|_F = \|\bm \Sigma_1  - \bm P_2\bm \Sigma_2\|_F \notag \\
   & \leq \frac{\sigma_{\min}^2(\bm U) + \|\bm V\|}{\sigma_{\min}^2(\bm U)} \|\bm \Sigma_1^2 - \bm P_1^T\bV\bV^T\bm P_1\|_F^2,
\end{align}
where the second equality follows from $\bm P_1^T\bV = \bm P_2\bm \Sigma_2\bm Q_2^T$ and $\bm Q_2^T \bQ = \bm Q_1^T$, and the inequality uses \Cref{ineq:tech-lemma}. Moreover, we have
$$
\|(\bm I - \bm P_1\bm P_1^T)\bV\bQ\|_F^2 = \|(\bm I - \bm P_1\bm P_1^T)\bV_1\|_F^2 \leq \frac{1}{\sigma_{\min}(\bm V)} \|(\bm I - \bm P_1\bm P_1^T)\bV_1\bV_1^T\|_F^2. 
$$
This, together with \eqref{eq3:lem UV}, \eqref{eq1:lem UV}, and \eqref{eq2:lem UV}, implies \eqref{eq:lem UV}. 
\end{proof}

\subsection{Proofs for Section \ref{subsec:noisy}: Noisy Case with General Sampling}\label{sec:app:noisy:general}

\subsubsection{Proof of Theorem \ref{thm:noisy:1}}\label{pf:thm 3}

\begin{proof}[Proof of Theorem \ref{thm:noisy:1}]
    Using the global optimality of $\bm U$, \eqref{model:UV}, and \eqref{eq:MC}, we have 
    \begin{align}\label{eq1:thm 3}
        F(\bm U) \leq F(\bU^\star) = \frac{1}{4} \|\bDelta_{\Omega}\|_F^2 \le \frac{1}{4} \|\bDelta\|_F^2.
    \end{align}
    In addition, we have
\begin{align*}
    F(\bm U) = \frac{1}{4}\|(\bm U \bm U^T - \bM)_\Omega\|_F^2 \geq \frac{1}{8}\|(\bm U \bm U^T - \bU^\star \bU^{{\star} T})_\Omega\|_F^2 - \frac{1}{4}\|\bDelta\|_F^2,
\end{align*}
where the inequality follows from \eqref{model:UV} and $\|\bm A - \bm B\|_F^2 \ge \|\bm A\|_F^2/2 - \|\bm B\|_F^2$ for all $\bm A, \bm B$. This, together with \eqref{eq1:thm 3}, implies 
\begin{align}
   \|(\bm U \bm U^T - \bU^\star\bU^{{\star} T})_\Omega\|_F^2 \leq 4\|\bDelta\|_F^2.  
\end{align}
Now, we consider the diagonal blocks of $\bm M$ index by $\mathcal{I}_k \times \mathcal{I}_k$. For each $k \in [K]$, let $\bm C_{k} := \bm U_k \bm U_k^T - \bU^\star_k \bU^{\star T}_k$. According to condition (a) in \Cref{ass:main}, the block $\mI_k \times \mI_k$ is fully observed, and thus we have for each $k \in [K]$,
\begin{align}\label{ineq:deltakk}
\|\bm C_k\|_F \leq \|(\bm U \bm U^T - \bU^\star\bU^{\star T})_\Omega\|_F \leq 2\|\bDelta\|_F.
\end{align}
Recall that $\bU_k^{\star^\dag} = (\bm U_k^{\star T}\bm U_k^\star)^{-1} \bm U_k^{\star T}$. Let $\bU_k^{\star^\dag} \bm U_k = \bP_k\bm \Sigma_k \bH_k^T$ be an SVD of $\bU_k^{\star^\dag} \bm U_k \in \R^{r\times r}$, where $\bm P_k, \bm H_k \in \mathcal{O}^r$ and $\bm \Sigma_k \in \R^{r\times r}$, and $\bQ_k := \bH_k \bP_k^T$. According to Lemma \ref{ineq:tech-lemma} and the proof, we have
 \begin{align}\label{ineq:tuq-bu}
     \|\bm U_k \bQ_k - \bU^\star_k\|_F  & \le  \frac{\sigma_{\min}(\bm U_k^\star) + \|\bm U_k\|}{\sigma_{\min}^2(\bm U_k^\star)} \|\bm U_k\bm U_k^T - \bm U_k^\star \bm U_k^{\star T}\|_F \notag \\
      & \leq \frac{2(\sqrt{\lambda} + \|\bm U_k\|)}{\lambda} \|\bm \Delta\|_F \leq \frac{2}{\lambda} \|\bDelta\|_F \left( \sqrt{\lambda} + 2\|\bDelta\|^{\frac{1}{2}} + \|\bU_k^\star\| \right),
 \end{align}
where the second inequality follows from \eqref{ineq:deltakk} and $\sigma_{\min}(\bm U_k^\star) \ge \sqrt{\lambda}$ for all $k \in [K]$ due to condition (b) in \Cref{ass:main}, and the last inequality uses
$$\|\bm U_k\| = \|\bm U_k\bm U_k^T\|^{\frac{1}{2}} = \|\bm C_k + \bU^\star_k \bU^{\star T}_k\|^{\frac{1}{2}} \leq \|\bm C_{k}\|^{\frac{1}{2}} + \|\bU^\star_k\| \le 2\|\bm \Delta\|^{\frac{1}{2}} + \|\bU^{\star}_k\| .$$ 
For these off-diagonal blocks satisfying $\mI_j \times \mI_k \in \Omega$, let 
$\bm C_{j, k} := \bm U_{j}\bm U_k^T - \bm U^\star_{j} \bm U^{\star T}_k$ for all $j \neq k \in [K]$. Similar to \eqref{ineq:deltakk}, we have $\|\bm C_{j, k}\|_F \leq 2\|\bDelta\|_F$ and $$
\bU_{j}^{\star^\dag} \bm U_{j} (\bU_k^{\star^\dag} \bm U_k)^T = \bI_r + \bU_{j}^{\star^\dag} \bm C_{j, k} \bU_k^{\star^\dag T}. 
$$
Substituting the SVDs $\bU_{j}^{\star^\dag} \bm U_{j} = \bP_{j}\bm \Sigma_{j} \bH_{j}^T$ and $\bU_k^{\star^\dag} \bm U_k =  \bP_k\bm \Sigma_k \bH_k^T$ into the above equality yields 
\begin{align*}
    \bI_r + \bU_{j}^{\star^\dag} \bm C_{j, k} (\bU_k^{\star^\dag})^T & =   \bU_{j}^{\star^\dag} \bm U_{j} (\bU_k^{\star^\dag} \bm U_k)^T = \bP_{j}\bm \Sigma_{j} \bH_{j}^T \bH_k\bm \Sigma_k\bP_k^T \\
    & =  \bP_{j}(\bm \Sigma_{j} -\bI_r) \bH_{j}^T \bH_k(\bm \Sigma_k - \bI_r) \bP_k^T + \bP_{j}(\bm \Sigma_{j} -\bI_r) \bH_{j}^T \bH_k \bP_k^T + \\
    &\qquad \bP_{j} \bH_{j}^T \bH_k(\bm \Sigma_k - \bI_r) \bP_k^T + \bP_{j}\bH_{j}^T \bH_k \bP_k^T
\end{align*}
Note that $\bQ_k = \bH_k \bP_k^T$ for all $k \in [K]$, and thus $\bP_{j}\bH_{j}^T \bH_k \bP_k^T = \bQ_{j}^T\bQ_k$. Then, we have
\begin{align}\label{ineq:qj-qk}
\begin{split}
\| \bQ_{j} - \bQ_k\|_F 
    & = \| \bQ_{j}\bQ_k^T - \bI_r \|_F  = \|\bP_{j}\bH_{j}^T \bH_k \bP_k^T-\bI_r\|_{F}\\
    & = \|\bU_{j}^{\star^\dag} \bm C_{j, k} (\bU_k^{\star^\dag})^T-\bP_{j}(\bm \Sigma_{j} -\bI_r) \bH_{j}^T \bH_k(\bm \Sigma_k - \bI_r) \bP_k^T -  \\
    &\qquad \bP_{j}(\bm \Sigma_{j} -\bI_r) \bH_{j}^T \bH_k \bP_k^T  - \bP_{j} \bH_{j}^T \bH_k(\bm \Sigma_k - \bI_r) \bP_k^T\|_F\\
    & \leq \|\bm \Sigma_{j} -\bI_r\|_F \|\bm \Sigma_k - \bI_r\|_F + \|\bm \Sigma_{j} -\bI_r\|_F + \|\bm \Sigma_k - \bI_r\|_F + \|\bU_{j}^{\star^\dag} \bm C_{j, k} \bU_k^{\star^\dag T}\|_F \\
    & \leq  \frac{1}{\lambda^2}\|\bm C_{j}\|_F \|\bm C_{k}\|_F + \frac{1}{\lambda}\|\bm C_{j}\|_F + \frac{1}{\lambda}\|\bm C_{k}\|_F + \frac{1}{\lambda}\|\bm C_{j, k}\|_F \\
    &  \leq \frac{2}{\lambda}\|\bDelta\|_F \left(3 + \frac{2}{\lambda}\|\bDelta\|_F\right), 
\end{split}
\end{align}
where the second inequality follows from \eqref{eqn:bnd4SmI}. For each $k \in[K]$, according to condition (d) in \Cref{ass:main}, there exists a path $k_0 \rightarrow k_1 \rightarrow \cdots \rightarrow k_s$ such that $k_0 = k$, $k_s = 1$ and $\mI_{k_{j-1}} \times \mI_{k_j} \subseteq \Omega$ for all $j \in [s]$. Therefore, we bound the term $\bQ_k -\bQ_1$ as follows:
\begin{align}\label{eq2:thm 3}
    \|\bQ_k -\bQ_1\|_F \leq \sum_{j\in [s]} \|\bQ_{k_{j-1}} - \bQ_{k_j}\|_F \leq \frac{2K}{\lambda} \|\bDelta\|_F \left(3 + \frac{2}{\lambda}\|\bDelta\|_F\right), 
\end{align}
where the last inequality follows from \eqref{ineq:qj-qk} and $s\leq K$. Therefore, we obtain 
\begin{align*}
    \|\bm U - \bU^\star\bQ_1^T\|_F^2 & = \sum_{k=1}^K \|\bm U_k - \bU^\star_k \bQ_1^T\|_F^2 \leq 2\sum_{k=1}^K \left(\|\bm U_k - \bU^\star_k \bQ_k^T\|_F^2 + \| \bU^\star_k (\bQ_k -\bQ_1)^T\|_F^2\right) \\
    & \leq \frac{8K}{\lambda^2} \|\bDelta\|_F^2 (\sqrt{\lambda} + 2\|\bDelta\|^{\frac{1}{2}} + \|\bU^\star\|)^2 +  2 K^2\|\bU^\star\|^2 \left(\frac{2}{\lambda}\|\bDelta\|_F \left(3 + \frac{2}{\lambda} \|\bDelta\|_F\right)\right)^2 \\
    & \leq \frac{8K}{\lambda^2} \|\bDelta\|_F^2 \left( \left(\sqrt{\lambda} + 2\|\bDelta\|_F^{\frac{1}{2}} + \|\bU^\star\| \right)^2 + K^2\|\bU^\star\|^2\left(3 + \frac{2}{\lambda}\|\bDelta\|_F\right)^2 \right)
\end{align*}
where the second inequality follows from \eqref{ineq:tuq-bu} and \eqref{eq2:thm 3}, and the third inequality holds because $\| \bU^\star_k (\bQ_k -\bQ_1)^T\|_F \leq \|\bU^\star_k\| \|\bQ_k -\bQ_1\|_F$. As a result, we get 
\begin{align}\label{ineq:tilde-star}
    \|\bm U - \bU^\star\bQ_1^T\|_F \leq  \frac{c_0}{\lambda} \|\bDelta\|_F, 
\end{align}
where 
\begin{align*}
c_0 := \sqrt{8K} \left(\sqrt{\lambda} + 2\|\bDelta\|_F^{\frac{1}{2}} + \|\bU^\star\| + K\|\bU^\star\| \left(3 + \frac{2}{\lambda} \|\bDelta\|_F\right) \right). 
\end{align*}
Finally, based on the fact that $\bU^\star \bU^{\star T} = \bU^\star \bQ_1^T (\bU^{\star} \bQ_1^T)^T$, we have
\begin{align}
    \begin{split}
        \|\bm U\bm U^T - \bU^\star \bU^{\star T} \|_F & = \|\bm U\bm U^T - \bU^\star \bQ_1^T (\bU^{\star} \bQ_1^T)^T \|_F\\
        & = \|(\bm U - \bU^\star \bQ_1^T)(\bm U - \bU^\star \bQ_1^T)^T  + (\bm U - \bU^\star\bQ_1^T)(\bU^{\star} \bQ_1^T)^T \\
        &\quad + \bU^\star \bQ_1^T (\bm U - \bU^\star \bQ_1^T)^T\|_F \\
        & \leq  \|\bm U - \bU^\star\bQ_1^T\|_F\left(2\|\bU^\star\| + \|\bm U - \bU^\star\bQ_1^T\|\right) \\
        & \leq \frac{c_0}{\lambda} \left(2\|\bU^\star\|+ \frac{c_0}{\lambda} \|\bDelta\|_F\right)\|\bDelta\|_F,
    \end{split}
\end{align}
where the first inequality holds because of the triangle inequality and the second inequality holds because of \eqref{ineq:tilde-star}. We complete the proof by letting 
\begin{align}\label{eq:c depe}
    c := \frac{c_0}{\lambda} \left(2\|\bU^\star\|_2+ \frac{c_0}{\lambda} \|\bDelta\|_F\right) \|\bDelta\|_F. 
\end{align}
\end{proof}

\subsubsection{Proof of Theorem \ref{thm:noisy:2}}\label{pf:thm 4}
\begin{proof}[Proof of Theorem \ref{thm:noisy:2}]
Recall that $\bm U \in \R^{n\times r}$ is a global optimal solution of Problem \eqref{eq:MC}. For ease of exposition, let $\tDelta := (\bm U \bm U^T - \bM)_{\Omega}$. Obviously, we have $\|\tDelta\|_F \leq \|\bDelta\|_F$. This, together with \eqref{eq:hess:direc}, yields 
\begin{align}\label{eq:hessuhat}
        \nabla^2F(\bm U)[\bm D, \bm D] & \geq  \frac{1}{2} \|(\bm U \bD^T + \bD\bm U^T)_\Omega\|_F^2 - \|\bm D\bm D^T\|_F \|\tDelta\|_F \notag\\
        &  \ge \frac{1}{2} \|(\bm U \bD^T + \bD\bm U^T)_\Omega\|_F^2 - \|\bDelta\|_F\|\bD\|_F^2.
    \end{align}
Next, we consider some direction $\bD$ such that $ (\bm U \bD^T + \bD\bm U^T)_\Omega = \bm 0$. We claim that $\bD = \bm U \bV$ for some skew-symmetric matrix $\bV\in \mathbb{R}^{r \times r}$. Indeed, for each $i \in [2^r]$, since the $i$-th diagonal part is fully observed, i.e., $\Omega_{i,i} = [n_i] \times [n_i]$, we have
\begin{align}\label{eq:uddu-noise}
    \bm U_i \bD^T_i + \bD_i\bm U^{T}_i = \bm 0.
\end{align}
Using Weyl's inequality, we obtain   
\[
\sigma_{r}(\bm U_i) \geq \sigma_{r}(\bU_i^\star) - \|\bm U_i\bQ_i - \bU_i^\star\|_F \geq \sqrt{\lambda}  - \frac{2}{\lambda} \|\bDelta\|_F \left(\sqrt{\lambda}  + \sqrt{\gamma} + \|\bU^\star\|\right) \geq \frac{1}{2}\sqrt{\lambda},
\]
where the second inequality follows from the assumption that $\bU^{\star T}_i\bU^\star_i \succeq \lambda \bI_r$, \eqref{ineq:tuq-bu}, and $\|\bDelta\|_F < \gamma/8$ by \eqref{ineq:noise-upperbound}, and the last inequality holds because of $\|\bm \Delta\|_F \le {\lambda\sqrt{\lambda}}/{4(\sqrt{\lambda} + \sqrt{\gamma} + \|\bU^\star\|)}$ by \eqref{ineq:noise-upperbound}. Therefore, the pseudo-inverse $\bm U_i^{\dag} = (\bm U^{T}_i \bm U_i)^{-1} \bm U^{T}_i$ is well-defined for all $i \in [2^r]$.
Multiplying $\bm U^{\dag T}_i$ from the right-hand side of \eqref{eq:uddu-noise} yields $\bD_i = - \bm U_i (\bm U_i^{\dag}\bD_i)^T = \bm U_i \bV_i$, where $\bm V_i = -(\bm U_i^{\dag}\bD_i)^T$.
Moreover, multiplying $\tU^{\dag}_i$ and $\tU_i^{\dag^T}$ from the left and right hand sides of \eqref{eq:uddu-noise} yields $(\bm U^{\dag}_i\bD_i)^T + \bm U_i^{\dag}\bD_i = 0$. This implies $\bV_i + \bV_i^T = \bm 0$, which means $\bV_i$ is a skew matrix. Then, consider the $(i,j)$-th block of $(\bm U \bD^T + \bD\bm U^{T})_\Omega$. Since $\bD_i = \bm U_i \bV_i$ for any $i\in [2^r]$, it follows from \eqref{eq:uddu-noise} that 
\begin{align*}
    \bm 0 = \bm U_{j} \bD^T_i + \bD_{j}\bm U^{T}_i =  \bm U_{j} \bV^T_i \bm U^{T}_i + \bm U_{j} \bV_{j} \bm U^{T}_i  = \bm U_{i+1} \left(\bV^T_i +  \bV_{i+1} \right) \bm U^{T}_i. 
\end{align*}
Since $\bm U_i$ and $\bm U_{i+1}$ are full-rank, we obtain $\bV^T_i +  \bV_{i+1} = \bm 0$. Thus, we have $\bV_{i} = \bV_{i+1}$ for all $i\in [2^r-1]$, and thus $\bV_i$ for all $i\in [2^r]$ are the same.

For ease of exposition, let $\mathcal{S} =  \mathcal{S}_{\bm U}$ and $\mathcal{S}^\perp =  \mathcal{S}^\perp_{\bm U}$. For any $\bD\in \mathbb{R}^{n \times r}$, we decompose it as
\begin{align}
    \bD = \bD_{\mathcal{S}} + \bD_{\mathcal{S}^\perp},\ \text{where}\ \bD_{\mathcal{S}} = {\rm proj}_{\mathcal{S}}(\bD), \, \bD_{\mathcal{S}^\perp} = {\rm proj}_{\mathcal{S}^\perp}(\bD).
\end{align}
Define the constant 
\begin{align}\label{defi:tildegamma}
    \widehat{\gamma} := \|(\bm U \tE^T + \tE\bm U^{T})_\Omega\|_F^2,\ \text{where}\ \tE = \mathop{\mathrm{argmin}}_{ \bE \in \mathcal{S}^{\perp}, \|\bE\|_F = 1} \| (\bm U\bE^T + \bE\bm U^T) _\Omega \|_F^2.
\end{align}
According to \eqref{eq:hessuhat}, we compute for any $\bD \in \mathcal{S}^\perp$,  
\begin{align*}
        \nabla^2F(\bm U)[\bm D, \bm D] & \ge  \frac{1}{2} \|(\bm U \bD_{\mathcal{S}^\perp}^T + \bD_{\mathcal{S}^\perp}\bm U^T)_\Omega\|_F^2  - \|\bDelta\|_F\|\bD\|_F^2 \\
        &  \geq \frac{\widehat{\gamma}}{2} \|\bD_{\mathcal{S}^\perp}\|_F^2  - \|\bDelta\|_F\|\bD\|_F^2 = \left(\frac{\widehat{\gamma}}{2} - \|\bDelta\|_F\right) \|\bD\|_F^2.
    \end{align*} 
We complete the proof by letting $\bm D = \mathrm{lift}_{\bm U} (\bm \xi)$ for any $\bm \xi \in \mathrm{T}_{[\bm U]}\mathcal{M}$ and showing that $\widehat{\gamma} \geq \gamma/4$. 

Then, the rest of the proof is devoted to proving $\widehat{\gamma} \geq \gamma/4$. Without loss of generality, we assume that $\bQ_1 = \bI_r$ in \eqref{ineq:tilde-star} and let $\oU = \bm U - \bU^\star$, then we have $\|\oU\|_F \le c_0\|\bDelta\|_F/\lambda \le \tilde{c}_0\|\bDelta\|_F/\lambda$ where
\begin{align}
    \tilde{c}_0 & = \sqrt{8K} \left(\sqrt{\lambda} + \sqrt{\gamma} + \|\bU^\star\| + K\|\bU^\star\| \left(3 + \frac{\gamma}{4\lambda} \right) \right) \notag \\
    & \geq  \sqrt{8K} \left(\sqrt{\lambda} + 2\|\bDelta\|_F^{\frac{1}{2}} + \|\bU^\star\| + K\|\bU^\star\| \left(3 + \frac{2}{\lambda} \|\bDelta\|_F\right) \right) = c_0, 
\end{align}
which holds because of the condition $\|\bDelta\|_F \leq \frac{\gamma}{8}$. Thus, we have 
\begin{align}\label{eq1:thm 4}
  \sqrt{\widehat{\gamma}} & \ge \|(\bU^\star \tE^T + \tE\bU^{\star T})_\Omega\|_F - \|(\oU \tE^T + \tE\oU^{T})_\Omega\|_F \notag \\
  & \ge  \|(\bU^\star \tE^T + \tE\bU^{\star T})_\Omega\|_F - \frac{2 \tilde{c}_0}{\lambda}\|\bDelta\|_F,  
\end{align}
where the first inequality holds because of the triangle inequality, and the second inequality follows from $\|\tE\|_F = 1$ and $\|\oU\|_F \le \tilde{c}_0\|\bDelta\|_F/\lambda$. Recall that $ \gamma $ is defined w.r.t. $\mathcal{S}^\perp _{\bU^\star} = \{ \bD: \bU^{\star T} \bD = \bD^T \bU^{\star} \}$, i.e., the subspace $\mathcal{S}^\perp _{ \bU^\star}$ is the solution space of the linear system $\bU^{\star T} \bD = \bD^T \bU^\star$. Even though $\tE$ may not fulfill the linear system requirements, due to the fact $\tE \in \mathcal{S}^\perp _{\bm U}$, we can bound 
$$ \|\bU^{\star T} \tE - \tE^T \bU^\star\|_F = \|\oU^{T} \tE - \tE^T \oU\|_F \leq \frac{2\tilde{c}_0}{\lambda}\|\bDelta\|_F, 
$$ 
where the first equality holds because of $\tE \in \mathcal{S}^\perp _{\bm U}$, i.e., $\bm U^T \tE = \tE^T \bm U$. According to Hoffman’s error bound for linear system \cite{hoffman2003approximate,pena2021new, guler1995approximations}, there exists some constant $\kappa >0$, which only depends on the linear system of $\bU^\star$, and some $\bD \in \mathcal{S}^\perp _{\bU^\star}$ such that \begin{align}\label{ineq:hoffman}
    \|\bD - \tE\|_F \leq \kappa \|\bU^{\star T} \tE - \tE^T \bU^\star\|_F \leq \frac{2\kappa \tilde{c}_0}{\lambda}\|\bm \Delta\|_F.
\end{align}
Then, we have 
\begin{align*}
\|(\bU^\star \tE^T + \tE \bU^{\star T}) _\Omega \|_F \ge & \|(\bU^\star \bD^T + \bD \bU^{\star T}) _\Omega\|_F - 2 \|\bU^\star\| \| \bD - \tE\|_F \\
\ge & \sqrt{\gamma} \|\bD\|_F -2\| \bU^\star\| | \bD - \tE\|_F \ge \sqrt{\gamma} - (\sqrt{\gamma} + 2\|\bU^\star\|) \|\bD - \tE\|_F,
\end{align*} 
where the first inequality holds because of the triangle inequality and the fact $\|\bU \bD\|_F \leq \|\bU\|\|\bD\|_F$ holds for any matrices $\bU$ and $\bD$, the second inequality follows from the definition of $\gamma$, and the last one uses the triangle inequality and $\|\tE\|_F = 1$. Substituting this and \eqref{ineq:hoffman} into \eqref{eq1:thm 4} yields 
$$ 
\sqrt{\widehat{\gamma}} \ge \|(\bU^\star \tE^T + \tE \bU^{\star T})_\Omega \|_F - \frac{2 \tilde{c}_0}{\lambda}\|\bDelta\|_F  \ge \sqrt{\gamma} - \frac{2 \tilde{c}_0 }{\lambda} \left( 1 + \kappa(\sqrt{\gamma} + 2\|\bU^\star\|_2)\right) \|\bDelta\|_F. $$ 
This, together with \eqref{ineq:noise-upperbound}, implies $\sqrt{\hat{\gamma}} \ge \sqrt{\gamma}/2$, which completes the proof.
\end{proof}

\subsubsection{Discussion on the relationship between $\gamma$ and $\lambda$}

Note that $\lambda$ is defined to satisfy $\bU_k^{\star T} \bU_k^\star = \sum_{i \in \mathcal{I}_k} \bu_i^\star(\bu_i^\star)^T \geq \lambda$. Under the setting of Gaussian matrix and noise, we could first compute the conditional expectation 
$$ \mathbb{E}\left[\bU_k^{\star T} \bU_k^\star \mid \bU_k^\star \in \mathcal{C}_k\right] = \sum_{i \in I_k} \mathbb{E} [\bu_i^\star(\bu_i^\star)^T \mid \bu_i^\star \in \mathcal{C}_k] = |\mathcal{I}_k| \mathbb{E} [\bu_i^\star(\bu_i^\star)^T \mid \bu_i^\star \in \mathcal{C}_k]. 
$$
Under the noiseless setting where $\mathcal{C}_k$ is an octant, a simple computation shows that 
\begin{align*}
    \mathbb{E}\left[\bu_i^\star(\bu_i^\star)^T \mid \bu_i^\star \in \mathcal{C}_k\right] = \left(1 - \frac{2}{\pi}\right) \bm I_r + \frac{2}{\pi} \bm e\bm e^T,
\end{align*}
where $\bm e\in \mathbb{R}^r$ is the all-one vector. For the noisy setting, even though it is not easy to get a close form, we could still show that $\mathbb{E}[\bu_i^\star(\bu_i^\star)^T \mid \bu_i^\star \in \mathcal{C}_k]$ is positive definite whose least eigenvalue is independent of $n$ and $r$, under the assumption $r = O(\log n)$. Besides, note that all $\bu_i^\star, i \in \mathcal{I}_k$ are still i.i.d. variables under the condition $\bu_i^\star \in \mathcal{C}_k$. By applying the concentration theory, we could show that $\lambda = \Theta(|I_k|) \geq \Theta(\sqrt{n})$, where the last inequality holds because we already shows that $|\mathcal{I}_k| \geq {\sqrt{n}}/{2}$ under the assumption $r = O(\log n)$.

Regarding $\gamma$, we define the linear map $\phi: \mathbb{R}^{n\times r} \mapsto \mathbb{R}^{n\times n}$ with $\phi(\bD) = (\bU^\star \bD^T + \bD\bU^{\star T})_{\Omega}$. The linear map $\phi$ can be vectorized by introducing a matrix $\bA \in \mathbb{R}^{n^2\times nr}$, i.e. $$ \mathrm{vec}(\phi(\bD)) = \bA\cdot \mathrm{vec}(\bD) $$ where $\bA$ consist of the entries of $\bU^\star$. Recall that $\gamma = \min ||\phi(\bD)||_F$ where $\bm D \in S^\perp, ||\bD||_F = 1 $. Since we already show $S^\perp$ is the kernel space of $\phi$, we know $\gamma$ is the smallest nonzero singular value of $\bA$. By carefully computation and applying the concentration inequality, one can also show $\gamma = O(\sqrt{n})$. Then, according to Equation (7), Theorem 3.8 holds when $\|\bDelta\|_F = O(n^{\frac{1}{4}})$, in other words, we need the noise level $\delta = O(n^{-\frac{3}{4}})$.

\subsection{Proofs for Section \ref{subsec:Gaus}: Verification of Assumptions under the Gaussian Distribution}\label{sec:app:gaussian}

\subsubsection{Proof of Proposition \ref{prop:gaussian-noiseless}}\label{pf:as 1}

\begin{proof}[Proof of \Cref{prop:gaussian-noiseless}]
    Since $\bu_i^\star \overset{i.i.d.}{\sim} \mathcal{N}(\bm 0,\bm I_d)$ for all $i \in [n]$ and $\mathcal{C}_k$ for each $k \in [2^r]$ is an orthant, we have
    \begin{align*}
    \P\left(\bu_i^\star \in \mathcal{C}_k\right) = \frac{1}{2^r}, \, \forall i \in [n], \forall k \in [2^r].
\end{align*}
Note that $n_k = \sum_{i=1}^n \bm{1}(\bu_i^\star \in \mathcal{C}_k)$ denotes the number of $\bm u_i^\star$ belonging to orthant $\mathcal{C}_k$. Applying the Bernstein inequality (see, e.g., \citep[Theorem 2.8.4]{vershynin2018high}) to $n_k$, we obtain 
\[
\P\left(n_k - \frac{n}{2^r} \leq t \right) \leq \exp\left(-\frac{\frac{1}{2}t^2}{\frac{n}{2^r} + \frac{t}{3}}\right), \forall\ t \geq 0.
\]
By choosing $t =  {n}/{2^{r+1}} \geq \sqrt{n}/{2}$ due to $r \leq  (\log_2 n)/2 $, we have for each $k \in [2^r]$,
\begin{align*}
    \P\left(n_k \leq \frac{n}{2^{r+1}}\right) \leq \exp\left(-\frac{\sqrt{n}}{16}\right)
\end{align*}
This, together with the union bound, yields 
\begin{align}\label{ineq:prob:nk}
    \P\left(n_k \geq \frac{\sqrt{n}}{2}, \forall  k \in [2^r] \right) \leq 1 - 2^r\exp\left(-\frac{\sqrt{n}}{16}\right)  \leq 1 - \sqrt{n} \exp\left(-\frac{\sqrt{n}}{16}\right).
\end{align}
Since $ {\sqrt{n}}/{2} \gg r $ and the entries of $\bU^\star$ is i.i.d. standard Gaussian random variables, Assumption \ref{ass:uistar} always holds.

Now, we consider the near diagonal block $\Omega_{k+1, k}$ for each $k\in [2^r -1]$. Without loss of generality, we assume $\mathcal{C}_k = \{(x_1, \cdots, x_r) : x_i\geq 0, i \in [r]\}$ and $\mathcal{C}_{k+1} = \{(x_1, \cdots, x_r) : x_i\geq 0, i \in [r-1], x_r \leq 0\}$. Therefore, for each row $\bu_i^\star \in \mathcal{C}_k$ and $\bu_j^\star \in \mathcal{C}_{k+1}$, $\mathrm{sgn}(\bu_i^\star)$ differs with $\mathrm{sgn}(\bu_j^\star)$ only at the last entry. Let $\mathcal{D}_k$ consist of all the vectors $\bm x \in \mathcal{C}_k$ satisfying $x_r \leq \max\{x_1, \cdots, x_{r-1}\}$, i.e., 
$$
\mathcal{D}_k = \left\{\bm x \in \mathcal{C}_k: x_r \leq \max\{x_1, \cdots, x_{r-1}\}\right\}.
$$
One can verify that the conditional probability $\P(\bu_i^\star \in \mathcal{D}_k \mid \bu_i^\star \in \mathcal{C}_k) \geq \frac{1}{2}$, and thus 
$$
\P(\bu_i^\star \in \mathcal{D}_k) = \P(\bu_i^\star \in \mathcal{D}_k \mid \bu_i^\star \in \mathcal{C}_k) \cdot \P( \bu_i^\star \in \mathcal{C}_k)  \geq \frac{1}{2^{r+1}}.
$$ 
By applying the concentration inequality, one can show that $|\mathcal{D}_k| \geq \sqrt{n}/{4}$ with high probability. Then, for each $\bu_i^\star \in \mathcal{D}_k$ and any $\bu_j^\star \in \mathcal{C}_{k+1}$, we have
$$
\P\left(\langle \bu_i^\star, \bu_j^\star\rangle \geq 0 \mid \bu_i^\star\right) \geq \frac{1}{2}.
$$
Note that
$$
|\Omega_{k+1, k}| = \sum_{(i,j) \in \Omega_{k+1, k}} \bm{1}(\bu_i^\star, \bu_j^\star\rangle \geq 0) \geq \sum_{i: \bu_i^\star \in \mathcal{D}_k} \sum_{j: \bu_j^\star \in \mathcal{C}_{k+1}} \bm{1}(\bu_i^\star, \bu_j^\star\rangle \geq 0)
$$
Conditioned on any $\bu_i^\star \in \mathcal{D}_k$, $\bm{1}(\bu_i^\star, \bu_j^\star\rangle \geq 0)$ for all $i,j$ are independent of each other. Therefore, we can apply the concentration inequality to show that
$$
\sum_{j: \bu_j^\star \in \mathcal{C}_{k+1}} \bm{1}(\bu_i^\star, \bu_j^\star\rangle \geq 0) \geq \frac{\sqrt{n}}{4}.
$$
As a result, $|\Omega_{k+1, k}| \geq n/16$ holds with high probability.
\end{proof}

\subsubsection{Proof of Proposition \ref{prop:guassian}}\label{pf:as 2}

For the noisy case, we need to first introduce the $\epsilon$-net to proceed.

\begin{definition}[Net and Covering Number]\label{defi:eps-net}
    Let $\mathcal{S}^{r-1} = \{\bu \in \mathbb{R}^{r} : \|\bu\| = 1\}$ be the unit sphere in $\mathbb{R}^{r}$ and $\epsilon > 0$ be a parameter. A subset $\mathcal{N}_{\epsilon} \subseteq \mathcal{S}^{r-1}$ is said to be an $\epsilon$-net of $\mathcal{S}^{r-1}$ if for any point $\bu \in \mathcal{S}^{r-1}$, there exists a point $\ba \in \mathcal{N}_{\epsilon}$ such that $\|\bu - \ba\| \leq \epsilon$. The cardinality of the smallest $\epsilon$-net is called the $\epsilon$-covering number, denoted by $N(\mathcal{S}^{r-1}, \epsilon)$. 
\end{definition}
According to \citet{szarek1997metric}, we know that, for any $\epsilon >0$,  
\begin{align}\label{ineq:covering-number}
    N(\mathcal{S}^{r-1}, \epsilon) \leq \left(1 + \frac{2}{\epsilon} \right)^r.
\end{align}
Let $\mathcal{N}_{\frac{1}{4}}$ denote the $\frac{1}{4}$-net $\mathcal{S}^{r-1}$ with the smallest cardinality. Then, \eqref{ineq:covering-number} implies that $K = |\mathcal{N}_{\frac{1}{4}}| \leq 9^r$. Let $\mathcal{N}_{\frac{1}{4}} = \{\ba_1, \ba_2, \dots, \ba_M\}$, then we define
\begin{align*}
    \mathcal{C}_k := \left\{\bu \in \mathbb{R}^r : \bu^T \ba_k \geq \frac{\sqrt{6} + \sqrt{2}}{4} \|\bu\| \right\}, \, \forall k \in [K].
\end{align*}
Since $\ba_k\in \mathcal{N}_{\frac{1}{2}} \subseteq \mathcal{S}^{r-1}$, the inequality $\bu^T \ba_k \geq ({\sqrt{6} + \sqrt{2}}) \|\bm u\| / 4$ means the angle between $\bu$ and $\ba_i$ is less than $\arccos\left(\frac{\sqrt{6} + \sqrt{2}}{4}\right) = \frac{\pi}{12}$. According to the definition of  $\epsilon$-net, we can show the following lemma. 
\begin{lemma}\label{lem:union-C}
    $\mathbb{R}^r = \bigcup_{k\in [K]} \mathcal{C}_k$.
\end{lemma}
\begin{proof} 
    For any $\bu \in \mathbb{R}^r \setminus \{\bm{0}\}$ and $ {\bu}/{\|\bu\|} \in \mathcal{S}^{r-1}$, there exists some $\ba_i \in \mathcal{N}_{\frac{1}{4}}$ such that $\|\ba_i - {\bu}/{\|\bu\|}\| \leq \frac{1}{4}$, which implies 
    \begin{align*}
        \left\langle \ba_i, \frac{\bu}{\|\bu\|} \right\rangle \geq \frac{31}{32} > \frac{\sqrt{6} + \sqrt{2}}{4}.
    \end{align*}
    This implies $\bu \in \mathcal{C}_i$. 
\end{proof}
We can now define the collection of index set as follows:
\begin{align*}
    \mI_k = \{i \in [n] : \bu^\star_i \in \mathcal{C}_k\}, \, \forall k \in [K],
\end{align*}
where $\bu^\star_i$ is the $i$-th row of $\bU^\star$. Lemma \ref{lem:union-C} implies that the collection of index sets $\{\mI_1, \mI_2, \dots, \mI_K\}$ satisfies condition $(a)$ in Assumption \ref{ass:main}. 

\begin{proof}[Proof of \Cref{prop:guassian}] 
First, Lemma \ref{lem:union-C} already implies condition (a). Next, for any $k \in [K]$ and any $(i, j) \in \mI_k \times \mI_k$, we know that $\bu^\star_i, \bu^\star_j \in \mathcal{C}_k$ according to the definition of $\mI_k$. Owing to the property of $\mathcal{C}_k$, we know that the angle between $\bu_i^\star$ and $\ba_k$ is less than ${\pi}/{12}$ and so is the angle between $\bu_j^\star$ and $\ba_k$. This implies that the angle between $\bu_i^\star$ and $\bu_j^\star$ is less than ${\pi}/{6}$. This further implies
\begin{align*}
    \bu_i^{\star T} \bu_j^{\star} \geq \cos\left(\frac{\pi}{6}\right) \|\bu_i^\star\| \|\bu_j^{\star T}\| = \frac{\sqrt{3}}{2}\|\bu_i^\star\| \|\bu_j^{\star T}\|.
\end{align*}
Then, we have $m_{ij} = \bu_i^{\star T} \bu_j^{\star} + \delta_{ij}\geq 0$ due to $\|\bDelta\|_\infty < \frac{1}{2}\min_{k\in [n]} \|\bu_k^\star\|_2^2$. It implies that $(i,j) \in \Omega$ for all $(i, j) \in \mI_k \times \mI_k$, i.e., condition (b) holds. 

Next, since each row $\bu_i^\star$ is an i.i.d. Gaussian vector, we have $ {\bu_i^\star}/{\|\bu_i^\star\|_2}$ follows the uniform distribution on the unit sphere $\mathcal{S}^{r-1}$. One can evaluate the probability that $\P(\bu_i^\star \in \mathcal{C}_k)$ in the following way: the area of $\mathcal{S}^{r-1}$ is equal to $\frac{2\pi^{r/2}}{\Gamma(r/2)} $, while the area of $\mathcal{S}^{r-1} \cap \mathcal{C}^k$ is bigger than volume of $(r-1)$-dimensional ball with radius equal to $\sin(\pi/12)$, which is $\frac{\pi^{(r-1)/2}}{\Gamma((r+1)/2)} \cdot \sin^{r-1}(\pi/12)$. This implies 
\begin{align*}
    \P (\bu_i^\star \in \mathcal{C}_k) \geq \frac{\frac{\pi^{(r-1)/2}}{\Gamma((r+1)/2)} \cdot \sin^{r-1}(\pi/12)}{\frac{2\pi^{r/2}}{\Gamma(r/2)} } = \frac{\Gamma(r/2)}{2\sqrt{\pi} \Gamma((r+1)/2)}  \sin^{r-1}(\pi/12) \geq \frac{1}{\sqrt{\pi}r} \sin^{r}(\pi/12),
\end{align*}
where the last inequality holds because $\Gamma((r+1)/2) = \frac{r-1}{2}\Gamma((r-1)/2)\leq \frac{r}{2}\Gamma(r)$ holds for all positive integers. Since $r\leq \log n/({4\log 3})$, we have
\begin{align*}
    \P\left(\bu_i^\star \in \mathcal{C}_k\right) \geq  \frac{1}{\sqrt{\pi}(r-1)}  \sin^{r-1}(\pi/12) \geq \frac{4\log 3}{\log n} \cdot n^{\frac{\log(\sin(\pi/12))}{4\log 3}} >  \frac{4\log 3}{\log n} n^{-1/3}.
\end{align*}
Similar to \eqref{ineq:prob:nk}, one can easily show that, with probability at least $1 - 1/n$, we have $|\mI_k| \geq  {\sqrt{n}}/{2}$ for all $k \in [K]$. Moreover, since $ \sqrt{n}/{2} \gg r$ and $\bU^\star$ is a matrix with i.i.d. Gaussian entries, it is easy to see that $\bU_k^{\star T} \bU_k^\star \succeq \lambda \bI_r$ holds for some $\lambda >0$.

Finally, for any two $\mathcal{C}_k$ and $\mathcal{C}_l$ satisfying $\mathcal{C}_k \cap \mathcal{C}_l \neq \emptyset$ and for any $\bu_i^\star \in \mathcal{C}_k$ and $\bu^\star_j \in \mathcal{C}_l$, let $\bu_0^\star \in \mathcal{C}_k \cap \mathcal{C}_l$, then the angle between $\bu_i^\star$ and $\bu_0^\star$ is less than ${\pi}/{6}$ and so is the angle between $\bu_j^\star$ and $\bu_0^\star$, which implies the angle between $\bu_i^\star$ and $\bu_j^\star$ is less than $\frac{\pi}{3}$. Thus, we have 
\begin{align*}
    \bu_i^{\star T} \bu_j^{\star} \geq \cos\left(\frac{\pi}{3}\right) \|\bu_i^\star\| \|\bu_j^{\star T}\| = \frac{1}{2}\|\bu_i^\star\| \|\bu_j^{\star T}\|.
\end{align*}
so $m_{ij} = \bu_i^{\star T} \bu_j^{\star} + \delta_{ij}\geq 0$ given the condition that $\|\bDelta\|_\infty < \min_{k\in [n]} \|\bu_k^\star\|^2 / 2$. It implies that $(i,j) \in \Omega$ for all $(i, j) \in \mI_k \times \mI_l$. For any $k, l \in [K]$, we can find a path $k_0 \rightarrow k_1 \rightarrow \cdots \rightarrow k_s$ such that $k_0 = k$, $k_s = l$ and $\mathcal{C}_{k_t} \cap \mathcal{C}_{k_{t+1}} \neq \emptyset$. This complete the whole proof.

\end{proof}

\section{Experimental Setups and Results}\label{sec:supp:expt}

\subsection{Additional results and figures} \label{sec:app:addexpt}

In this section, we provide the full version of \cref{sample-table}, which includes noise level $\sigma = 10^{-2}$. See Table \ref{sample-table-full}. We also provide an additional figure along the lines of \cref{fig:time}, also including noise level $\sigma = 10^{-2}$. In Figure \ref{fig:time-supp}, we see that all the algorithms perform very similarly in terms of completion accuracy with $\sigma = 10^{-2}$, and their computation time is similar to that for the noiseless case. 

\vspace{-0.1in}
\begin{table*}[t]
\caption{Comparison of the norm of gradient, function value, and completion error returned by GD with different initialization.}\label{table:2}
\label{sample-table-full}
\vskip -0.15in
\begin{center}
\begin{tabular}{lcccr}
\toprule
$\sigma = 0$ & $\alpha_{T}$ & $\beta_{T}$ & $\gamma_{T}$ \\
\midrule
{\bf Our init}  & $(9.7 \pm 0.16)\cdot 10^{-7}$ & $(3.3 \pm 0.36) \cdot 10^{-14}$ &  $ (7.6 \pm 0.7) \cdot 10^{-11}$ \\
\text{RI init} & $(6.8 \pm 0.3)\cdot 10^{-3}$ & $ (7.9 \pm 2.1) \cdot 10^{3}$ & $0.5 \pm 0.13$ \\ 
\text{RS init} & $0.1 \pm 0.45$ & $ (7.1 \pm 3.8) \cdot 10^{3} $ & $0.45 \pm 0.24$ \\
\midrule\midrule
$\sigma = 10^{-4}$ & $\alpha_{T}$ & $\beta_{T}$ & $\gamma_{T}$ \\
\midrule
{\bf Our init}    & $(9.7 \pm 0.14)\cdot 10^{-7}$ & $(1.8 \pm 0.03)\cdot 10^{-4}$ & $(4.4 \pm 0.13)\cdot 10^{-5}$ \\
{RI init} & $(0.39 \pm 1.2)\cdot 10^{-5}$ & $(8.4 \pm 2.1) \cdot 10^3$  & $0.51 \pm 0.13$ \\
\text{RS init} & $(1.5 \pm 6.5)\cdot 10^{-3}$ & $(8.4 \pm 2.0) \cdot 10^3$ & $0.51 \pm 0.13$ \\ 
\midrule\midrule
$\sigma = 10^{-2}$ & $\alpha_{T}$ & $\beta_{T}$ & $\gamma_{T}$ \\
\midrule
{\bf Our init}    & $(9.7 \pm 0.17)\cdot 10^{-7}$ & $1.8 \pm 0.03$ & $(4.4 \pm 0.22) \cdot 10^{-3}$ \\
{RI init} & $(8.1 \pm 0.35) \cdot 10^{-5}$ & $(7.3 \pm 3.3) \cdot 10^3$ & $0.46 \pm 0.2$ \\
\text{RS init} & $(6.4 \pm 0.19)\cdot 10^{-5}$ & $(7.9 \pm 3.6) \cdot 10^3$ & $0.418 \pm 0.21$ \\ 
\bottomrule
\end{tabular}
\end{center}
\vskip -0.1in
\end{table*}

\begin{figure*}[t]
\begin{center}
	\begin{subfigure}{0.48\textwidth}
    	\includegraphics[width = 1\linewidth]{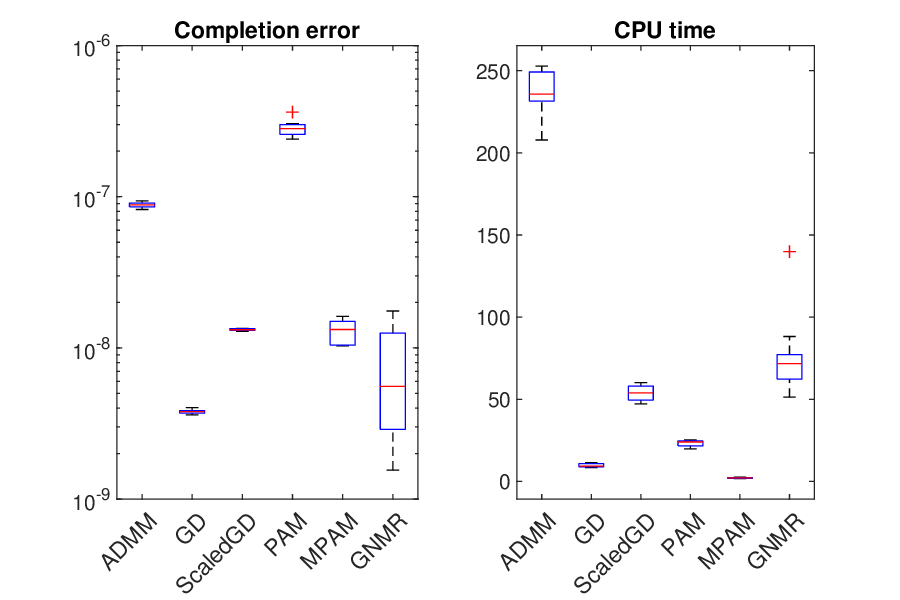}\vspace{-0.05in}
    \caption{Noiseless case: $\sigma = 0$} 
    \end{subfigure} 
    \begin{subfigure}{0.48\textwidth}
    	\includegraphics[width = 1\linewidth]{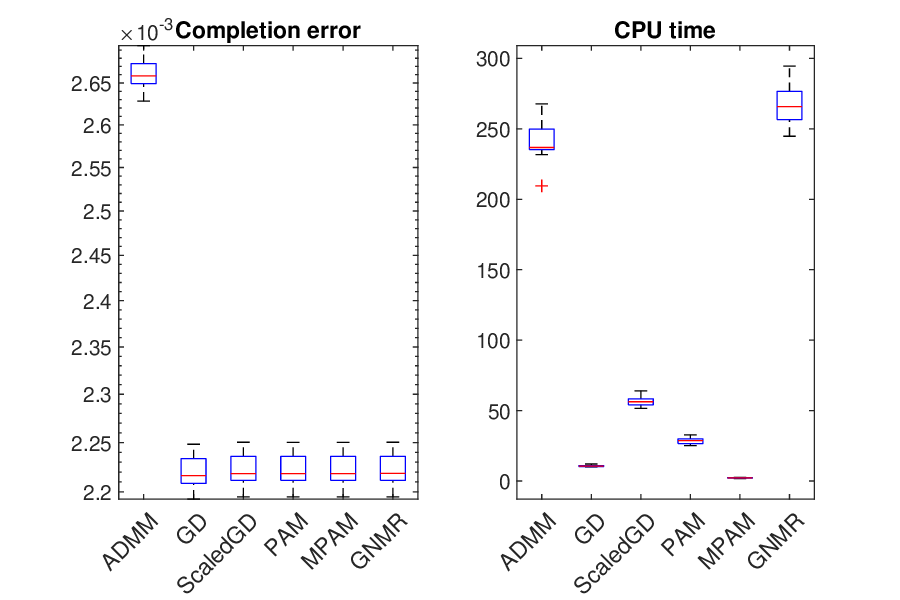}\vspace{-0.05in}
    \caption{Noisy case: $\sigma = 10^{-2}$} 
    \end{subfigure}\vspace{-0.1in} 
    \caption{\textbf{Completion error and CPU time of the tested methods for MC with ReLU sampling.}}\vspace{-0.15in} 
    \label{fig:time-supp}
\end{center}
\end{figure*}

\subsection{Experimental Setup for \Cref{subsec:comp}} \label{subsec:effi}

In this subsection, we give details of the studied methods in \Cref{subsec:comp}. 

\vspace{-0.1in}
\paragraph{ADMM for the convex formulation of MC.} Noting that the noisy model in \eqref{model:UV} is considered in this work, then we study the following convex formulation to complete the missing entries as studied in \cite{candes2010matrix}: 
\begin{align}\label{eq:cvx}
    \min_{\bm X \in \R^{n\times n}} \|\bm X\|_*\qquad \mathrm{s.t.}\ \left\| \bm X_{\Omega} - \bm M_{\Omega} \right\|_F^2 \le \delta. 
\end{align}
In our implementation, we simply set $\delta = \|\bm \Delta\|_F^2$. It is obvious that the classic formulation in \cite{recht2011simpler} is a special case of the above formulation when $\delta = 0$ in the noiseless case.  Then, one can efficiently solve this convex problem using ADMM. By introducing an auxiliary variable $\bm Y$, we first rewrite Problem \eqref{eq:cvx} as
\begin{align}
     \min_{\bm X \in \R^{n\times n}} \|\bm Y\|_*\qquad \mathrm{s.t.}\ \bm X = \bm Y,\ \left\| \bm X_{\Omega} - \bm M_{\Omega} \right\|_F^2 \le \delta. 
\end{align}
By introducing a dual variable $\bm \Lambda \in \R^{n\times n}$, one can write the augmented Lagrangian as follows:
\begin{align}
    \mathcal{L}(\bm X, \bm Y; \bm \Lambda) =  \|\bm Y\|_* + \delta_{\mathcal{X}}(\bm X) - \langle \bm \Lambda, \bm X - \bm Y \rangle + \frac{\rho}{2}\|\bm X - \bm Y\|_F^2,
\end{align}
where $\mathcal{X} =  \{\bm X \in \R^{n\times n}:\  \left\| \bm X_{\Omega} - \bm M_{\Omega} \right\|_F^2 \le \delta \}$. In the implementation, we randomly genearte an initial point $(\bm Y^{(0)}, \bm \Lambda^{(0)})$, whose entries are i.i.d. sampled from the standard normal distribution. Given the current iterate $(\bm X^{(t)}, \bm Y^{(t)}, \bm \Lambda^{(t)})$, ADMM generates the next iterate as follows: 
\begin{itemize}
    \item To update $\bm X$, we compute 
\begin{align*}
    \bm X^{(t+1)} = \mathrm{argmin}_{\bm X } \mathcal{L}(\bm X, \bm Y^{(t)}; \bm \Lambda^{(t)}). 
\end{align*}
By letting $\bm W^{(t)} = \bm Y^{(t)} + \bm \Lambda^{(t)} / \rho$, we compute its closed-form solution as follows: 
\begin{align*}
   & \bm X^{(t+1)} = \bm W^{(t)},\ \text{if}\ \| \bm W^{(t)}_{\Omega}  - \bm M_{\Omega} \|_F^2 \le \delta, \\
   & \bm X^{(t+1)}_{\Omega^C} = \bm W^{(t)}_{\Omega^C},\ \bm X^{(t+1)}_{\Omega} = \frac{\bm W^{(t)}_{\Omega} + t\bm M_{\Omega}}{1+t},\ \text{where}\ t = \frac{\|\bm W^{(t)}_{\Omega} - \bm M_{\Omega}\|_F^2}{\delta} - 1, \ \text{otherwise}. 
\end{align*}
\item  To update $\bm Y$, we compute 
\begin{align*}
    \bm Y^{(t+1)} = \mathrm{argmin}_{\bm Y} \mathcal{L}(\bm X^{(t+1)}, \bm Y; \bm \Lambda^{(t)}) =  \mathrm{argmin}_{\bm Y} \frac{1}{2}\left\| \bm Y - \left( \bm X^{(t+1)} - \bm \Lambda^{(t)}/\rho \right) \right\|_F^2 + \frac{1}{\rho}\|\bm Y\|_*,
\end{align*}
This is to compute the proximal mapping of the nuclear norm that admits a closed-form solution. 
\item To update $\bm \Lambda$, we have
\begin{align}
    \bm \Lambda^{(t+1)} = \bm \Lambda^{(t)} - \rho\left(\bm X^{(t+1)} - \bm Y^{(t+1)}\right)
\end{align}
\end{itemize}
We terminate the algorithm when $\left\| \bm X^{(t)} - \bm Y^{(t)} \right\|_F \le 10^{-4}$ for some $t \ge 0$.


\vspace{-0.1in}
\paragraph{ScaledGD for the formulation in \cite{tong2021accelerating}.} According to \cite{tong2021accelerating}, one can solve the following formulation for MC: 
\begin{align}\label{eq:MC1}
    \min_{\bm L, \bm R \in \R^{n\times r}} H(\bm L, \bm R) = \frac{1}{2}\left\|(\bm L \bm R^T)_{\Omega} - \bm M_{\Omega} \right\|_F^2. 
\end{align}
Here, we directly use their MATLAB codes downloaded from \url{https://github.com/Titan-Tong/ScaledGD} to implement ScaledGD for solving the above problem. Notably, we employ our tailor-designed initialization in \Cref{subsec:init} to initialize ScaledGD, which demonstrates great performance in our experiments. We terminate the algorithm when $\left\|  \nabla H(\bm L^{(t)}, \bm R^{(t)}) \right\|_F \le 10^{-4}$ for some $t \ge 0$.  

\vspace{-0.1in}
\paragraph{PAM for the formulation in \cite{saul2022nonlinear}.} To complete sparse nonnegative matrices with low-dimensional structures, \citet{saul2022nonlinear} studied the following formulation: 
\begin{align}\label{eq:sparse}
    \min_{\bm X, \bm \Theta \in \R^{n\times n}} \|\bm X - \bm \Theta\|_F^2\qquad \mathrm{s.t.}\ \mathrm{rank}(\bm \Theta) = r,\ \bm X_{\Omega} = \bm M_{\Omega},\ \bm X_{\Omega^c} \le 0. 
\end{align}
Since the variables $\bm X, \bm \Theta$ are separable in this problem, we can directly proximal alternating minimization (PAM) to solve this problem. In the implementation, we randomly genearte an initial point $(\bm X^{(0)}, \bm \Theta^{(0)})$, whose entries are i.i.d. sampled from the standard normal distribution. Given the current iterate $(\bm X^{(t)}, \bm \Theta^{(t)})$, PAM generates the next iterate as follows: 
\begin{align*}
    & \bm X^{(t+1)} \in \mathrm{argmin}_{\bm X}  \|\bm X - \bm \Theta^{(t)}\|_F^2 + \frac{\alpha}{2}\|\bm X - \bm X^{(t)}\|_F^2 \qquad \mathrm{s.t.}\ \bm X_{\Omega} = \bm M_{\Omega},\ \bm X_{\Omega^c} \le 0,\\
    & \bm \Theta^{(t+1)} \in \mathrm{argmin}_{\bm \Theta}  \|\bm X^{(t+1)} - \bm \Theta\|_F^2 + \frac{\beta}{2}\|\bm \Theta - \bm \Theta^{(t)}\|_F^2 \qquad \mathrm{s.t.}\ \mathrm{rank}(\bm \Theta) = r. 
\end{align*}
Obviously, the above subproblems both admit closed-form solutions. We terminate the algorithm when $\|\bm X^{(t+1)} - \bm X^{(t)}\|_F +  \|\bm \Theta^{(t+1)} - \bm \Theta^{(t)}\|_F \le 10^{-4}$ for some $t \ge 0$. 

\vspace{-0.1in}
\paragraph{Momentum PAM for the NMD formulation in \cite{seraghiti2023accelerated}.} Noting that $\bm \Theta$ in Problem \eqref{eq:sparse} admits the low-rank structure, one can naturally reformulate this problem into the following nonlinear matrix decomposition (NMD) formulation: 
\begin{align*} 
    \min_{\bm X \in \R^{n\times n}, \bm W, \bm H \in \R^{n\times r}} \|\bm X - \bm W \bm H\|_F^2\qquad \mathrm{s.t.}\ \ \bm X_{\Omega} = \bm M_{\Omega},\ \bm X_{\Omega^c} \le 0. 
\end{align*}
In particular, \citet{seraghiti2023accelerated} have proposed a momentum PAM method for solving this problem. Here, we directly use their MATLAB codes downloaded from \url{https://gitlab.com/ngillis/ReLU-NMD} to implement momentum PAM for solving this problem. We terminate the algorithm when the relative error $\|\bm X^{(t)} - (\bm W^{(t)} \bm H^{(t)})_{\Omega}\|_F/\|\bm X^{(t)}\|_F \le 10^{-4}$. 

\vspace{-0.1in}
\paragraph{Gaussian-Newton matrix recovery (GNMR) for low-rank MC \cite{zilber2022gnmr}.} \citet{zilber2022gnmr} employed the Gauss-Newton linearization to design a GNMR method for solving the non-convex formulation of MC. Here, we directly use the MATLAB codes in \cite{naik2022truncated} to implement this method for solving MC with ReLU sampling.

\end{document}